\documentclass[11pt]{article} 
\usepackage{tikz}
\usepackage{eqnarray,amsmath,amsfonts,amsthm,mathrsfs}
\usepackage{color}
\usepackage{bm}
\usepackage{amssymb}
\usepackage{graphicx}
\usepackage{multirow,booktabs}
\usepackage{float}
\usepackage{fancyhdr}
\usepackage{subcaption}
\usepackage{enumerate}
\usepackage[top=2.5cm, bottom=2.5cm, left=3cm, right=3cm]{geometry}
\usepackage{booktabs,longtable,authblk}
\usepackage{mathrsfs,hhline}
\usepackage{hyperref}
\usepackage[ruled,linesnumbered]{algorithm2e}
\usepackage[numbers,sort&compress]{natbib} 
\newtheorem{theorem}{Theorem}

\newtheorem{definition}{Definition}

\newtheorem{lemma}{Lemma}

\newtheorem{remark}{Remark}

\allowdisplaybreaks[4]

\title{Simultaneous CNN Approximation on Manifolds with Applications to Boundary Value Problems$^\dag$\footnotetext{\dag~The work of Lei Shi is partially supported by the National Natural Science Foundation of China (Grant No.12571099). The corresponding author is Lei Shi.}}

\author[1]{Hanfei Zhou}
\author[1,2]{Lei Shi}

\affil[1]{School of Mathematical Sciences, \linebreak
Fudan University, Shanghai, 200433, China 
}
\affil[2]{
Shanghai Key Laboratory for Contemporary Applied Mathematics, \linebreak
Fudan University, Shanghai, 200433, China \linebreak
Email:zhouhf23@m.fudan.edu.cn, leishi@fudan.edu.cn
}
\date{}
\begin{document}
\maketitle

\begin{abstract}
This paper develops convolutional neural network (CNN) methods for simultaneous Sobolev approximation and elliptic boundary value problems on compact Riemannian manifolds. We prove approximation estimates for single- and multichannel CNNs, with rates governed by the intrinsic dimension and the smoothness gap. Motivated by elliptic stability, we propose a physics-informed CNN framework with a spectral boundary loss. The boundary residual is expanded in boundary Laplace--Beltrami eigenmodes and penalized by Sobolev trace weights, matching the natural \(\mathcal H^{2s-1/2}(\partial\mathcal M^d)\) trace norm for \(2s\)-order elliptic problems. This avoids smooth auxiliary constructions for exact boundary enforcement and singular Sobolev--Slobodeckij double integrals, while allowing FFT-based or precomputed spectral implementations. We also derive an error decomposition separating approximation, generalization, and spectral truncation errors, showing that the proposed loss is aligned with localized fast-rate generalization analysis. Numerical experiments on the upper hemisphere and upper half-torus demonstrate improved accuracy, convergence, and stability over standard PINNs, with one to two orders of magnitude gains for high-frequency boundary data.
\end{abstract}

\section{Introduction}\label{Section: Introduction}
Convolutional neural networks (CNNs) have been highly successful in image analysis~\cite{krizhevsky2012imagenet}. A common explanation is the manifold hypothesis, namely, that relevant data are concentrated near low-dimensional manifolds embedded in high-dimensional ambient spaces~\cite{bengio2013representation,mallat2016understanding}. In this work, we consider the use of CNNs for solving partial differential equations (PDEs) posed on manifolds, with boundary conditions incorporated into the formulation. Because PDEs involve differential operators, their numerical treatment requires approximation estimates in Sobolev norms rather than merely in standard $L_p$ norms. This leads naturally to the study of simultaneous approximation by CNNs on compact manifolds.

Let $(\mathcal M^d,g)$ be a compact Riemannian manifold with $\dim \mathcal M^d=d$. We write $\mathcal W^{k,p}(\mathcal M^d)$ for the Sobolev space of order $k$ and integrability exponent $p$ on $\mathcal M^d$. We first consider the following simultaneous approximation problem in lower-order Sobolev norms: for $f\in \mathcal{W}^{k,p}(\mathcal M^d)$, we construct a CNN $u\in\mathcal F$ such that
\begin{equation*}
    \|f-u\|_{\mathcal{W}^{s,p}(\mathcal M^d)}\le \varepsilon,
    \qquad  0\le s<k,
    \end{equation*}
where $\mathcal F$ denotes a class of CNNs specified by architectural parameters such as width, depth, and the number of nonzero parameters. At the same time, we characterize the complexity of such approximations by quantifying how the required network complexity depends on the tolerance $\varepsilon>0$, the regularity indices $(k,s,p)$, and the intrinsic dimension $d$.

We next introduce the CNN architectures considered in this paper. Throughout, we work with one-dimensional CNNs and consider two classes of architectures, namely, single-channel and multichannel CNNs, for both of which we establish simultaneous approximation results in Sobolev norms. We begin with the single-channel setting. For the sphere, simultaneous approximation results for single-channel CNNs were obtained in~\cite{lei2025solving}. However, that analysis relies essentially on the zonal structure of spherical harmonics. Although the same idea can be extended, via the addition formula~\cite{evarist1975addition}, to two-point homogeneous spaces~\cite{wang1952two}, this class of spaces has been completely classified and is too narrow to cover general compact manifolds. We then turn to the multichannel setting. This is not merely a technical refinement: multichannel CNNs are closer to modern deep learning architectures and, according to our theory, admit sharper control of the parameter complexity. The effectiveness of this multichannel construction is also supported by our numerical results. 

From the viewpoint of approximation theory, the corresponding theory for simultaneous approximation by CNNs on manifolds remains limited. Existing quantitative results for convolutional architectures, including~\cite{zhou2020universality,zhou2020theory}, are largely restricted to Euclidean domains. Our approach instead builds CNN constructions from restricted kernels on manifolds. A main analytical difficulty is the high-order approximation of such kernels, since they exhibit a singularity at the origin. To exploit their analyticity away from the origin, we introduce a truncation near the singular point and construct separate approximations for the singular and analytic parts.

Our first contribution is a quantitative simultaneous approximation theory for $\mathrm{ReLU}$--$\mathrm{ReQU}$ CNNs on general compact manifolds. In~\autoref{theorem: simultaneous approximation}, we show that for every $f\in \mathcal W^{k,p}(\mathcal M^d)$ and every $\varepsilon>0$, there exists a deep single-channel CNN $u_\varepsilon$ such that
\begin{equation*}
    \|f-u_\varepsilon\|_{\mathcal W^{s,p}(\mathcal M^d)} \lesssim \varepsilon, \qquad 0\le s<k,
\end{equation*}
with total parameter count $N \lesssim \varepsilon^{-\frac{d}{k-s}}.$ We also prove in~\autoref{thm:mult-inner-products} that constant-depth multichannel CNNs enjoy the same simultaneous approximation property, with sharper bounds on the network complexity.

From the viewpoint of nonlinear approximation theory, the rate $N\lesssim \varepsilon^{-\frac{d}{k-s}}$ agrees with the classical lower bounds arising from entropy and width estimates~\cite{pinkus2012n,devore1998nonlinear,siegel2024sharp}. Hence, for general Sobolev classes on manifolds, the scaling $\varepsilon^{-d/(k-s)}$ is sharp. However, for single-channel CNNs, the depth necessarily grows polynomially in $\varepsilon^{-1}$. As a consequence, complexity estimates for the corresponding hypothesis class, for instance covering number bounds, become suboptimal. This in turn deteriorates the theoretical convergence rate of the overall algorithm, so that minimax-optimal performance cannot be achieved within this architecture. This provides a further motivation for passing to the multichannel setting, where the constant-depth construction allows for sharper control of the network complexity. Moreover, in both the single-channel and multichannel settings, the CNN is augmented by a very shallow fully connected tail. Our complexity estimates show that the convolutional layers dominate the overall parameter count, while the fully connected part contributes only lower-order terms.

Building on the above approximation theory, we next consider physics-informed neural networks (PINNs) for PDEs. A central issue is the mismatch between the low-order boundary penalties in standard PINNs and the trace norms required for PDE stability. Consider the even-order boundary value problem
\begin{equation}\label{PDE}
\begin{cases}
\mathcal L u = f, & \text{in }\mathcal{M}^d,\\
u = g, & \text{on }\partial\mathcal{M}^d,
\end{cases}
\end{equation}
with $\mathcal L$ a uniformly elliptic operator of order $2s$. By standard elliptic a priori estimates on manifolds~\cite{narcowich2017novel}, stability in $\mathcal W^{2s}_p(\mathcal M^d)$ requires boundary control in $\mathcal W^{2s-1/p}_p(\partial\mathcal M^d)$. In the Hilbert case $p=2$, this becomes $\mathcal H^{2s-1/2}(\partial\mathcal M^d)$, rather than the $\mathcal L^2(\partial\mathcal M^d)$ penalty typically used in PINNs. This mismatch is a basic source of instability and slow convergence.

The above mismatch has also become apparent from the viewpoint of convergence analysis. 
Recent studies on consistent PINNs show that error estimates in the natural PDE stability norm require loss terms that dominate the Sobolev norms appearing in the corresponding elliptic a priori estimates~\cite{bonito2026convergence,mishra2026consistent}. 
In particular, even for second-order elliptic problems, boundary discrepancies should be controlled in an $H^{1/2}(\partial\Omega)$-type trace norm rather than merely in an $L^2(\partial\Omega)$ norm. 
For higher-order elliptic equations, this requirement becomes more pronounced, since the boundary data naturally live in higher-order trace spaces.

Many existing analyses avoid this difficulty by considering boundary-free settings, homogeneous boundary conditions, or neural network classes that satisfy the boundary condition exactly. However, exact boundary enforcement is not automatic for neural networks. It is often implemented through geometry-aware trial functions based on approximate distance functions or lifting constructions~\cite{sukumar2022exact,berrone2023enforcing}. While effective for many Dirichlet problems, this strategy shifts the difficulty to constructing sufficiently smooth distance or lifting functions. For complex geometries, nonsmooth boundaries, higher-order PDEs, or derivative boundary conditions, the required regularity of such auxiliary functions can become a serious limitation.

A different line of work incorporates Sobolev-type boundary penalties directly into the PINN loss. For instance, recent trace-regularity PINNs enforce fractional boundary norms such as $H^{1/2}(\partial\Omega)$ through Sobolev--Slobodeckij type double integrals~\cite{kim2025trace,zhou2025ssbe}. Such formulations are closer to the stability theory of elliptic PDEs, but the nonlocal double-integral form involves singular kernels near the diagonal and may be expensive or delicate to discretize, especially on curved boundaries or manifolds.

To address this issue, we employ a physics-informed CNN framework with a boundary treatment adapted to the trace norms required by PDE stability. The high-order boundary penalty is computed through boundary inner products and frequency weights, rather than through singular nonlocal kernels or repeated high-order automatic differentiation. The resulting loss has a frequency-resolved structure: boundary errors in different modes are penalized according to their Sobolev trace weights, and a computable high-order penalty is obtained by truncating the expansion to finitely many boundary modes. On boundaries with explicit or precomputable spectral bases, such as circles or product-type boundaries, this formulation converts the trace norm into a diagonal frequency energy and can exploit fast transforms such as Fast Fourier Transforms (FFTs).

This construction is also useful for analysis. We derive an error decomposition for the corresponding PICNN estimator, separating the solution error into approximation, generalization, and spectral truncation terms (see~\autoref{theorem:picnn_error_decomposition}). The approximation term is controlled by the preceding simultaneous CNN approximation theory, while the spectral form of the boundary loss is aligned with localized fast-rate generalization analysis. Thus the Sobolev boundary loss is not only an algorithmic device for improving numerical performance, but also provides a loss structure better matched to elliptic stability estimates.

Our experiments on the upper hemisphere and the upper half-torus confirm these effects. For smooth variable-coefficient elliptic problems, the spectral Sobolev boundary penalty consistently improves both relative \(L^2\) and \(H^2\) accuracy and yields smoother convergence trajectories compared with the standard \(L^2\) boundary penalty. The advantage becomes more pronounced for high-frequency Dirichlet data. The spectral boundary loss substantially suppresses spurious Fourier modes and achieves one to two orders of magnitude improvements in the high-frequency regime. A cutoff sweep further shows that a moderate number of boundary modes already provides stable performance, revealing a practical trade-off between spectral resolution and empirical variance.

The paper is organized as follows. In~\autoref{Section: Preliminaries}, we state the standing assumptions on the manifolds and review Sobolev norm characterizations. We then introduce restricted kernels on manifolds, the associated reproducing kernel Hilbert space (RKHS), and the CNN architecture. ~\autoref{Section: Main Results} presents the main simultaneous approximation results. ~\autoref{Section: Implementation of Spectral-Trace PINNs} introduces the spectral-trace PICNN formulation for boundary value problems and provides an error decomposition for the resulting estimator.~\autoref{Section: Experimental} reports numerical experiments validating the proposed method. ~\autoref{Section: Conclusion} concludes the paper and discusses future directions. Detailed proofs of the main theorems and auxiliary lemmas are provided in the~\hyperref[allapp]{Appendix}.

\section{Preliminaries}\label{Section: Preliminaries}
\subsection{Assumptions on Manifolds}
Let $\mathcal{M}^d$ be a compact, connected, $d$-dimensional Riemannian manifold. We assume bounded geometry (see~\autoref{Appendix: Notations}), which ensures the existence of a finite smooth atlas $\{U_j, \psi_j\}_{j=1}^K$ consisting of geodesic balls $U_j = B(m_j, r)$ with coordinate maps defined by the inverse exponential map $\psi_j = \exp_{m_j}^{-1}$ \cite{de2021reproducing}. Associated with this atlas is a subordinate smooth partition of unity $\{\rho_j\}_{j=1}^K$ such that $\mathrm{supp}(\rho_j) \subset U_j$ and $\sum_{j=1}^K \rho_j = 1$. While the following approximation theory focuses on boundaryless manifolds to streamline the presentation, the underlying geometric framework naturally extends to manifolds with boundaries.

\subsection{Sobolev Functions on Manifolds}\label{Subsection: Sobolev Functions on Manifolds}
Sobolev spaces on manifolds can be characterized either locally via coordinate charts or globally via the Laplace–Beltrami spectral decomposition. We adopt the spectral framework, leveraging its connection to RKHS and heat kernel analysis to derive CNN approximation estimates based on established norm equivalences for $p \ge 1$.
\begin{theorem}[Theorem 3 in~\cite{de2021reproducing}]\label{theorem: sobolev equivalence}
Based on the geometric assumptions on $\mathcal{M}^d$ established above, the following Sobolev norms are equivalent for any $s > 0$:
\begin{enumerate}[(i)]
    \item Local Chart Norm: Defined via the atlas $\{(U_j,\psi_j)\}$ and partition of unity $\{\rho_j\}$:
    \begin{equation}\label{equation: sobolev definition1}
             \|f\|_{\mathcal{W}^{s}_2,1}^2 := \sum_{j} \bigl\|(\rho_j f)\circ \psi_j^{-1}\bigr\|_{W^{s}_2(\mathbb{R}^d)}^2.
    \end{equation}
    \item Bessel Potential Norm: Defined via the spectral calculus of the Laplacian:
    \begin{equation}\label{equation: sobolev definition2}
        \|f\|_{\mathcal{W}^{s}_2,2} := \|(I+\Delta_{\mathcal{M}})^{s/2} f\|_{L^2(\mathcal{M}^d)}.
    \end{equation}
    \item Operator Graph Norm:
    \begin{equation}\label{equation: sobolev definition3}
        \|f\|_{\mathcal{W}^{s}_2,3}^2 := \|f\|_{L^2(\mathcal{M}^d)}^2 + \|\Delta_{\mathcal{M}}^{s/2} f\|_{L^2(\mathcal{M}^d)}^2.
    \end{equation}
\end{enumerate}
Furthermore, if $s \in \mathbb{N}_+$, these norms are equivalent to the Riemannian Sobolev norm defined by covariant derivatives:
\begin{equation}\label{equation: sobolev definition4}
    \|f\|_{\mathcal{W}^s_2,4}^2 := \sum_{\ell=0}^s \|\nabla^\ell f\|_{L^2(\mathcal{M}^d)}^2.
\end{equation}
\end{theorem}
Hereafter, we denote the manifold Sobolev space $\mathcal{W}^k_2(\mathcal{M}^d)$ by $\mathcal{H}^k(\mathcal{M}^d)$, and the Euclidean Sobolev space ${W}^k_2(\mathbb{R}^D)$ by $H^k(\mathbb{R}^D)$.

\subsection{The CNN Architectures}\label{subsection: The CNN Architectures}
In this subsection, we formally define the $\mathrm{ReLU}$-$\mathrm{ReQU}$ CNN architectures investigated in this work. We consider two distinct configurations: a single-channel model characterized by an expanding spatial width, and a multichannel model with a fixed spatial width. Conceptually, the initial $\mathrm{ReLU}$ convolutional layers serve to capture the local geometric structure of the manifold, while the subsequent $\mathrm{ReQU}$ fully connected layers aggregate global information and enforce the sufficient regularity required for simultaneous approximation.
\paragraph{1-D Single-Channel Architecture.}
We first adopt the 1-D single-channel CNN architecture proposed in~\cite{lei2025solving}. This architecture is characterized by $L$ layers of convolutions where the network width expands at each step. The recurrence relations are governed by:
\begin{equation}\label{CNN layer}
    \begin{aligned}
        F^{(0)}:\mathbb{R}^d \to \mathbb{R}^d, \quad &F^{(0)}(x) = x; \\
        F^{(l)}:\mathbb{R}^d \to \mathbb{R}^{d_l},  \quad 
        &F^{(l)}(x) = \sigma^{(l)}\left( \left( \sum_{j=1}^{d_{l-1}} w^{(l)}_{i-j} \big(F^{(l-1)}(x)\big)_j 
        \right)_{i=1}^{d_l} - b^{(l)}\right),\\ 
        &l=1,\ldots,L.
    \end{aligned}
\end{equation}
Here, $\sigma^{(l)}$ denotes the activation function. The convolutional kernels $\{ w^{(l)} \}_{l=1}^{L}$ have support size $S^{(l)} \ge 3$ (i.e., $\mathrm{supp}(w^{(l)}) = \{ 0, \ldots ,S^{(l)}-1 \}$). The bias vectors are $b^{(l)} \in \mathbb{R}^{d_l}$, and the network width expands according to $d_{l}=d_{l-1}+S^{(l)}-1$, initialized with $d_0=d$. This convolution operation admits an equivalent matrix-vector representation:
$$
 \left( \sum_{j=1}^{d_l-1}w^{(l)}_{i-j}\left(F^{(l-1)}(x)\right)_j\right)_{i=1}^{d_l} = T^{(l)}F^{(l-1)}(x),
$$
where $T^{(l)} \in \mathbb{R}^{d_l \times d_{l-1}}$ is a Toeplitz matrix characterized by the kernel weights:
$$
T^{(l)} = \begin{bmatrix}
 w_{0}^{(l)}& 0 & 0 & \cdots & 0 \\ 
  w_{1}^{(l)}	 & w_{0}^{(l)} & 0 & \cdots & 0  \\
\vdots & \vdots & \vdots & \ddots & \vdots \\
w_{S^{(l)}-1}^{(l)}	 & \cdots & w_{0}^{(l)} & \cdots & 0 \\  
\vdots & \ddots & \ddots & \ddots & \vdots \\
0 & \cdots & w^{(l)}_{S^{(l)}-1} & \cdots & w^{(l)}_0 \\
\vdots & \vdots & \vdots & \ddots & w^{(l)}_{S^{(l)}-1}
\end{bmatrix}.
$$
Transitioning to the fully connected stage, we introduce a downsampling operator $\mathcal{D}: \mathbb{R}^{d_L}\to \mathbb{R}^{\lfloor d_{L}/d \rfloor}$ defined by $\mathcal{D}(x)_i = x_{i \cdot d}$. This is followed by $L_0$ fully connected layers and a final linear output layer:
\begin{equation}\label{FCNN layer}
    \begin{aligned}
        F^{(L+1)}&:\mathbb{R}^d \to \mathbb{R}^{d_{L+1}},\\
        &F^{(L+1)}(x) = \sigma^{(L+1)}\left( W^{(L+1)}\mathcal{D}\big(F^{(L)}(x)\big) - b^{(L+1)}\right);\\
        F^{(l)}&:\mathbb{R}^d \to \mathbb{R}^{d_l}, \\
        &F^{(l)}(x) = \sigma^{(l)}\left( W^{(l)}F^{(l-1)}(x) - b^{(l)}\right), \quad l=L+2,\ldots,L+L_0; \\
        F^{(L+L_0+1)}&:\mathbb{R}^d \to \mathbb{R},\\
        &F^{(L+L_0+1)}(x) = W^{(L+L_0+1)} F^{(L+L_0)}(x) - b^{(L+L_0+1)}.
    \end{aligned}
\end{equation}
The complexity of the network is quantified by the total number of free parameters, denoted by $\mathcal{S}$, which is calculated by accounting for weight sharing in the convolutional layers and sparsity in the connections. The hypothesis space is formally defined as the class of functions realizable by this architecture under the specified constraints:
\begin{equation}\label{equation: Assumption for hypothesis space}
         \begin{aligned}
             &\mathcal{F}_{1-\mathrm{channel}}(L, L_0, S, d_{L+1}, \ldots, d_{L+L_0}, \mathrm{ReLU}, \mathrm{ReQU}, \mathcal{S}) \\
             ={}& \left\{ 
             F^{(L+L_0+1)}: \mathbb{R}^{d} \to \mathbb{R} \left\vert
             \begin{array}{l}
                 \text{$F^{(L+L_0+1)}$ follows \eqref{CNN layer}-\eqref{FCNN layer};} \\
                 \text{$S^{(l)} = S, \sigma^{(l)} = \mathrm{ReLU}$ for $l \neq L+1$;}\\
                 \text{$\sigma^{(l)} = \mathrm{ReQU}$ for $l = L+1$ (with $k \geq s$); } \\
                 \text{Total free parameters} \le \mathcal{S}.
             \end{array}
             \right.
             \right\}.
         \end{aligned}
\end{equation}
For brevity, we denote it as $\mathcal{F}_{1-\mathrm{channel}}$.
\paragraph{Multichannel Architecture.}
We now introduce a multichannel structure in the convolutional stage following~\cite{yang2025rates}. In this setting, the spatial width $D$ remains fixed throughout the convolutional layers, while multiple channels are used to enrich the representation.

Let \(w=(w_1,\dots,w_s)^\top\in\mathbb R^s\) be a filter of size \(s\in[D]\). 
We define the associated one-sided convolution matrix \(T_w\in\mathbb R^{D\times D}\) by
\[
T_w :=
\begin{pmatrix}
w_1 & \cdots & w_{s-1} & w_s &        &        &  \\
    & \ddots & \ddots  & \ddots & \ddots &        &  \\
    &        & w_1 & \cdots & w_{s-1} & w_s &  \\
    &        &     & w_1 & \cdots & w_{s-1} & w_s \\
    &        &     &     & \ddots & \ddots & \vdots \\
    &        &     &     &        & w_1 & \cdots \\
    &        &     &     &        &     & w_1
\end{pmatrix}\in\mathbb R^{D\times D},
\]
where all blank entries are zero. This corresponds to one-sided padding and stride-one convolution. Let \(J,J'\in\mathbb N\) denote the input and output channel sizes, respectively. For a filter tensor $w=(w_{i,j',j})_{i\in[s],\,j'\in[J'],\,j\in[J]}\in\mathbb R^{s\times J'\times J}$ and a bias vector \(b=(b_1,\dots,b_{J'})^\top\in\mathbb R^{J'}\), we define the convolutional layer $\mathrm{Conv}_{w,b}:\mathbb R^{D\times J}\to\mathbb R^{D\times J'}$ by
\begin{equation*}
    \bigl(\mathrm{Conv}_{w,b}(x)\bigr)_{:,j'}
    :=
    \sum_{j=1}^{J} T_{w_{:,j',j}}\,x_{:,j}
    + b_{j'}\mathbf 1_D,
    \qquad j'=1,\dots,J',
\end{equation*}
for $x=(x_{i,j})_{i\in[D],\,j\in[J]}\in\mathbb R^{D\times J}$, where \(\mathbf 1_D\in\mathbb R^D\) denotes the constant vector with all entries equal to \(1\). 

We now define the multichannel CNN.  Let \(s\in[D]\), \(J,L,m\in\mathbb N\) be the filter size, channel size, depth, and output dimension, respectively. The CNN part is parameterized by $\bigl(w^{(0)},b^{(0)},\dots,w^{(L-1)},b^{(L-1)}\bigr)$ and takes the form
\begin{equation*}
    \begin{aligned}
        F_{\mathrm{Mult}}^{(0)}&:\mathbb R^D\to\mathbb R^{D\times 1},\,
        F_{\mathrm{Mult}}^{(0)}(x)=x,\\
        F_{\mathrm{Mult}}^{(\ell)}&:\mathbb R^{D\times J}\to\mathbb R^{D\times J'},\,
        F_{\mathrm{Mult}}^{(\ell)}(x)= \sigma^{(\ell)}
        \Bigl(
        \mathrm{Conv}_{w^{(\ell-1)},b^{(\ell-1)}}\bigl(F_{\mathrm{Mult}}^{(\ell-1)}(x)\bigr)
        \Bigr), \, \ell=1,\dots,L,
    \end{aligned}
\end{equation*}
and the final output takes the following form:
\begin{equation*}
\Bigl(
\langle W^{(L,1)},F_{\mathrm{Mult}}^{(L)}(x)\rangle_F,,
\dots,,
\langle W^{(L,m)},F_{\mathrm{Mult}}^{(L)}(x)\rangle_F
\Bigr)^\top,
\end{equation*}
Here, $\sigma^{(\ell)}$ denotes the $\mathrm{ReLU}$ activation function for $\ell \in [L]$, and $\langle\cdot,\cdot\rangle_F$ represents the Frobenius inner product. The parameter dimensions are specified as $(w^{(0)}, b^{(0)}) \in \mathbb{R}^{s\times J\times 1} \times \mathbb{R}^J$ for the initial layer, and $(w^{(\ell)}, b^{(\ell)}) \in \mathbb{R}^{s\times J\times J} \times \mathbb{R}^J$ for the hidden convolutional layers $\ell \in [L-1]$. The outputs generated by the Frobenius inner products with $W^{(L,r)} \in \mathbb{R}^{D \times J}$ ($r \in [m]$) are subsequently fed into a sequence of $L_0$ fully connected layers, following the architecture defined in~\eqref{FCNN layer}. We formally denote the class of functions realizable by this hybrid architecture as:
\begin{equation*}
    \mathcal{F}^{\mathrm{fixed}}_{\mathrm{Mult}}(J, m, L, L_0, S, d_{L+1}, \ldots, d_{L+L_0}, \mathrm{ReLU}, \mathrm{ReQU}, \mathcal{S}).
\end{equation*}
The meanings and constraints of the parameters  within this notation are consistent with those established in our previous definition in~\eqref{equation: Assumption for hypothesis space}. We simplify the notation to $\mathcal{F}_{\mathrm{Mult}}$ for brevity.

\section{Main Simultaneous Approximation Results}\label{Section: Main Results}
In this section, we present the main results for single-channel and multichannel CNNs, including simultaneous approximation estimates and complexity bounds.
\begin{theorem}\label{theorem: simultaneous approximation}
Let $\mathcal{M}^d$ satisfy the assumptions in~\autoref{Section: Preliminaries}, and let $f \in \mathcal{H}^k(\mathcal{M}^d)$ with $k = \tau -(D-d)/2$. For any error tolerance $\varepsilon > 0$ and smoothness index $0 \leq s < k$, there exists a CNN $F^{(L+L_0+ 1)} \in \mathcal{F}_{1-\mathrm{channel}}$ with convolutional depth
\begin{equation*}
    L \lesssim \left\lceil \frac{n_2\varepsilon^{-d/(k-s)}D-1}{\,S-2\,}\right\rceil, \quad \text{and fully connected depth} \quad L_0 \lesssim \log \log(1/\varepsilon),
\end{equation*}
width
    \begin{equation*}   
        d_l \lesssim \varepsilon^{-d/(k-s)}\log^2(1/\varepsilon), \quad l=L+1,\ldots,L+L_0+1,
    \end{equation*}
and total number of parameters
\begin{equation*}
    \mathcal{S} \lesssim \varepsilon^{-d/(k-s)},
\end{equation*}
such that
\begin{equation*}
    \|f - F^{(L+L_0+1)}\|_{\mathcal{H}^s(\mathcal{M}^d)} \lesssim \varepsilon.
\end{equation*}
Here, $S\in\mathbb{N}$ (with $S\ge 3$) is the filter size used in all convolutional layers, and $n_2$ is the constant from~\autoref{lemma: feature construction}.
\end{theorem}

\begin{remark}
    \autoref{theorem: simultaneous approximation} gives simultaneous approximation rates for Sobolev functions on manifolds by CNNs; see \autoref{Proof: Proof of simultaneous approximation with 1 channel} for the proof. The complexity depends only on the intrinsic dimension $d$, not on the ambient dimension $D$, thereby mitigating the curse of dimensionality. In addition, the fully connected depth grows only like $\log\log(1/\varepsilon)$, so the resulting architecture remains very shallow. The estimates further suggest that width, rather than depth, is the primary driver of approximation power.
\end{remark}

\begin{remark}
    It is worth noting that the $O(\log\log \varepsilon^{-1})$ depth of our construction differs from the finite-depth approximation results usually available for $\mathrm{ReQU}$ networks applied to globally smooth functions. The reason is that kernels such as the Mat\'ern class have only limited Sobolev regularity because of their singularity at the origin. Under standard parallel architectures, this would force the parameter complexity to scale polynomially in $\varepsilon^{-1}$ and lead to saturation of the approximation rate. Our construction overcomes this difficulty by exploiting the piecewise analyticity of the kernel. In this way, the network depth converts the potential polynomial complexity into logarithmic complexity, while preserving narrow width and sharp Sobolev error control.
\end{remark}

We now turn to the multichannel construction, which is motivated by an inherent limitation of the one-channel architecture. For one-channel 1-D CNNs with increasing width, the covering number (or VC-dimension) estimate scales as $SL\log S$; see~\cite[Theorem 3]{lei2025solving}. Since in that construction $S$ and $L$ are of the same order, this leads, up to logarithmic factors, to an effective complexity term of order $S^2\log S$. Such a bound is not sharp and may result in suboptimal generalization rates. This issue was already observed in~\cite{lei2025solving}. A sharper result in this architecture should therefore avoid any polynomial dependence on the depth in the corresponding complexity factor. The obstruction comes from two sources. First, the one-channel construction forces a bias explosion, which already prevents a sharp entropy estimate. Second, the wavelet-type filter factorization used there is intrinsically unstable. In order to obtain a depth-sharp bound, one would need at least a stability estimate of the form $\prod_{\ell=1}^{L}\|w^{(\ell)}\|_{1}\ \lesssim\ C\|W\|_{1}\cdot \mathrm{poly}(L),$ for a convolutional factorization $W=w^{(L)}*\cdots * w^{(1)}.$
However, the rough bound available in the previous results~\cite[Lemma 5]{lei2025solving}, obtained through the quantity $B^{(L)}$, only yields an exponential dependence on $L$, and in general falls far short of such a polynomial control. According to our construction, the matrices that generate inner-product features typically do not satisfy such a stability estimate. For this reason, we introduce below a multichannel construction, which avoids these obstructions and leads to a genuinely sharper complexity estimate.
\begin{theorem}\label{thm:mult-inner-products}
Under the same setting as~\autoref{theorem: simultaneous approximation}, let \(S\in[2:D]\) be the filter size, there exists a multichannel CNN $F^{(L+L_0+1)} \in \mathcal{F}_{\mathrm{Mult}}$ with convolutional depth
\begin{equation*}
    L=\Big\lceil \frac{D-1}{S-1}\Big\rceil, \quad \text{and fully connected depth} \quad  L_0 \lesssim \log\log(1/\varepsilon),
\end{equation*}
width
\begin{equation*}   
    d_l \lesssim \varepsilon^{-d/(k-s)}\log^2(1/\varepsilon), \quad l=L+1,\ldots,L+L_0+1,
\end{equation*}
channel size
\begin{equation*}
    J \lesssim \varepsilon^{-d/(k-s)}
\end{equation*}
and total number of parameters

\begin{equation*}
    \mathcal{S} \lesssim \varepsilon^{-d/(k-s)},
\end{equation*}
such that
\begin{equation*}
    \|f - F^{(L+L_0+1)}\|_{\mathcal{H}^s(\mathcal{M}^d)} \lesssim \varepsilon.
\end{equation*}
\end{theorem}

\begin{remark}
    Within the present multichannel framework, the network is written as $f=h\circ g$, where $g$ is the ReLU-CNN block and $h$ is the terminal fully connected ReQU network. If $\mathcal F=\{h\circ g:\ g\in\mathcal G,\ h\in\mathcal H\}$ and each $h\in\mathcal H$ is $\Lambda$-Lipschitz on a bounded set $K$ containing the range of $g\in\mathcal G$, then a standard composition argument yields $\log \mathcal N(\delta,\mathcal F,\|\cdot\|_{L^\infty})\le \log \mathcal N(\delta/(2\Lambda),\mathcal G,\|\cdot\|_{L^\infty})+\log \mathcal N(\delta/2,\mathcal H,\|\cdot\|_{L^\infty(K)})$. Thus, the ReQU head affects the entropy bound only through the rescaled accuracy $\delta/(2\Lambda)$. Since the CNN block has constant depth, its entropy contribution remains mild, unlike the blow-up that may occur in one-channel width-increasing architectures. The only additional issue is the local Lipschitz bound of $h$: because $\sigma(t)=(t)_+^2$ has derivative depending on the input size, one has $\log\Lambda\lesssim 2^{L_0}$ on bounded sets. However, since $L_0\asymp \log\log(\varepsilon^{-1})$, this contributes only a polylogarithmic factor in $\varepsilon^{-1}$ and hence has only a mild effect on the final complexity bound.
\end{remark}

\begin{remark}
    We also mention the regime of finitely many channels with increasing spatial width. In this case, similar approximation results are expected to hold by directly constructing the required features, without the wavelet-type factorization used in the one-channel architecture. The CNN depth would then remain of the same order as in~\autoref{theorem: simultaneous approximation}. It is plausible that, combined with the feature-allocation idea in~\cite[Lemma A.5]{yang2025rates}, one could further distribute the feature-induced weights across depth and obtain a complexity bound analogous to~\cite[Theorem 2.5]{yang2025rates}, with only logarithmic dependence on the depth. We leave this question for future work.
\end{remark}

\section{PICNN for Boundary Value Problems}\label{Section: Implementation of Spectral-Trace PINNs}
In this section, we apply the preceding approximation theory to elliptic boundary value problems on manifolds. We consider~\eqref{PDE}, where \(\mathcal L\) is a uniformly elliptic operator of order \(2s\). Standard PINNs typically impose an \(L^2\)-type boundary penalty, while classical elliptic stability theory indicates that controlling
\(\|u-u^*\|_{\mathcal H^{2s}(\mathcal M^d)}\) requires boundary control in the sharp trace space
\(\mathcal H^{2s-\frac12}(\partial\mathcal M^d)\).

To address this mismatch, we introduce a physics-informed CNN framework with a spectral boundary penalty. The boundary residual is expanded in the eigenbasis of the boundary Laplace--Beltrami operator and penalized according to Sobolev trace weights. This yields a computable approximation of the \(\mathcal H^{2s-\frac12}(\partial\mathcal M^d)\) trace norm without singular Sobolev--Slobodeckij integrals or repeated high-order automatic differentiation.

Let $\mu$ be the normalized volume measure on $\mathcal M^d$, and let $\sigma$ be the normalized surface measure on $\partial\mathcal M^d$. Their empirical counterparts are
denoted by $  \mu_n:=\frac1n\sum_{i=1}^n\delta_{x_i}, \sigma_m:=\frac1m\sum_{j=1}^m\delta_{y_j}.$
For an integrable function \(q\), we use the notation $\mu q:=\int_{\mathcal M^d}q\,d\mu, \sigma q:=\int_{\partial\mathcal M^d}q\,d\sigma,$ and similarly $    \mu_n q:=\frac1n\sum_{i=1}^n q(x_i), \sigma_m q:=\frac1m\sum_{j=1}^m q(y_j).$

For \(u\in\mathcal F\), set $e_u:=u|_{\partial\mathcal M^d}-g,\beta:=2s-\frac12 .$ Let \(\{(\psi_k,\lambda_k)\}_{k\ge 1}\) be the \(L^2(\sigma)\)-orthonormal eigenpairs of the Laplace--Beltrami operator on \(\partial\mathcal M^d\). For \(k=1,\ldots,K\), define the population and empirical boundary spectral coefficients by $c_k(u):=\sigma(e_u\psi_k),\widehat c_k(u):=\sigma_m(e_u\psi_k).$ The truncated population spectral boundary energy is defined by $    \mathcal S_K(u) := \sum_{k=1}^K \lambda_k^\beta |c_k(u)|^2,$ whereas its empirical counterpart is $\widehat{\mathcal S}_K(u) :=\sum_{k=1}^K \lambda_k^\beta |\widehat c_k(u)|^2.$ Accordingly, define the truncated population functional
\[
    \mathcal J_K(u)
    :=
    \mu|\mathcal Lu-f|^2
    +
    \sigma|e_u|^2
    +
    \mathcal S_K(u),
\]
and the empirical functional
\begin{equation}\label{equation: modified_pinn_loss}
    \begin{aligned}
    \widehat{\mathcal J}_K(u)
    &:= \mu_n|\mathcal Lu-f|^2 + \sigma_m|e_u|^2 + \widehat{\mathcal S}_K(u)\\
    &= \frac{1}{n}\sum_{i=1}^n |\mathcal{L}u(x_i)-f(x_i)|^2 + \frac{1}{m}\sum_{j=1}^m |u(y_j)-g(y_j)|^2 \\
    &+ \sum_{k=1}^K \lambda_k^{2s-\frac12}\left|
        \frac{1}{m}\sum_{j=1}^m (u(y_j)-g(y_j))\psi_k(y_j)
    \right|^2.
    \end{aligned}
\end{equation}
The PICNN estimator is defined by $\widehat u_{\mathcal F,K}=\arg\min_{u\in\mathcal F}\hat{\mathcal{J}}_{K}(u)$, where $\mathcal F$ is the CNN hypothesis class. This loss is designed to capture the high-order boundary regularity required by the PDE without introducing excessive computational cost, especially from repeated automatic differentiation. 

Before presenting the numerical experiments, we first provide a theoretical error analysis for the proposed PICNN method. The following theorem gives an error decomposition for the estimator trained with the empirical spectral boundary functional. It separates the solution error into an approximation error, a generalization error, and a spectral truncation error. Moreover, the spectral form of the boundary loss is well suited to localized empirical-process estimates, and therefore provides a natural route toward fast rate control of the generalization term~\cite{bartlett2005local,lei2025solving,lu2022machine,zhou2025weak}.

Define the full trace functional
\begin{equation*}
        \mathcal J_\infty(u) := \mu|\mathcal Lu-f|^2 + \sigma|e_u|^2 + \sum_{k=1}^{\infty} \lambda_k^\beta \left| \langle e_u,\psi_k\rangle_{L^2(\sigma)} \right|^2,
\end{equation*}
and the spectral truncation tail $ \tau_K(\mathcal F) := \sup_{u\in\mathcal F} \sum_{k>K} \lambda_k^\beta \left| \langle e_u,\psi_k\rangle_{L^2(\sigma)} \right|^2 .$ We also introduce the spectral weight density $\Theta_K(y):=\sum_{k=1}^K\lambda_k^\beta\psi_k(y)^2,$ and its uniform bound $\Gamma_K :=\|\Theta_K\|_{L^\infty(\partial\mathcal M^d)}=\sup_{y\in\partial\mathcal M^d}\sum_{k=1}^K \lambda_k^\beta\psi_k(y)^2 .$

\begin{theorem}
\label{theorem:picnn_error_decomposition}
Assume that the boundary value problem admits a unique solution \(u^*\), and that the
elliptic stability estimate
\[
    \|u-u^*\|_{\mathcal H^{2s}(\mathcal M^d)}^2
    \le
    C_{\mathrm{st}}\mathcal J_\infty(u)
\]
holds for every admissible \(u\in\mathcal F\). Let $u_{\mathcal F,K} \in \arg\min_{u\in\mathcal F}\mathcal J_K(u).$ Assume that, for all \(u\in\mathcal F\), $    |\mathcal Lu-f|\le M_I, |e_u|\le M_B .$ Then for every \(t\ge 1\), with probability at least \(1-Ce^{-t}\),
\[
\begin{aligned}
\|\widehat u_{\mathcal F,K}-u^*\|_{\mathcal H^{2s}(\mathcal M^d)}^2
\le C \Bigg[
&\left(2+\frac{\Gamma_K}{m}\right)\mathcal J_K(u_{\mathcal F,K})
+\tau_K(\mathcal F)
\\
&\quad+G_K(\widehat u_{\mathcal F,K})+\frac{M_I^2 t}{n}+\frac{M_B^2(1+\Gamma_K)t}{m}
\Bigg],
\end{aligned}
\]
where $ G_K(\widehat u_{\mathcal F,K}) := \mathcal J_K(\widehat u_{\mathcal F,K}) - \widehat{\mathcal J}_K(\widehat u_{\mathcal F,K}) .$ Here \(C>0\) depends only on the stability and geometric constants, but is independent of \(n,m\), and \(K\).
\end{theorem}

\begin{remark}
The proof is deferred to~\autoref{appendix: error decomposition}. The estimate separates the error into three components. The first term \(\mathcal J_K(u_{\mathcal F,K})\) is the approximation error induced by the hypothesis class \(\mathcal F\). Under the Lipschitz continuity of \(\mathcal L\) with respect to the \(\mathcal H^{2s}\)-norm and the trace theorem, this term is controlled by the simultaneous Sobolev approximation error of the CNN class, as quantified in~\autoref{theorem: simultaneous approximation} and~\autoref{thm:mult-inner-products}. The empirical-process term satisfies
\[
    G_K(\widehat u_{\mathcal F,K})
    =
    \mathcal J_K(\widehat u_{\mathcal F,K})
    -
    \widehat{\mathcal J}_K(\widehat u_{\mathcal F,K})
    \le
    \sup_{u\in\mathcal F}
    \left\{
        \mathcal J_K(u)-\widehat{\mathcal J}_K(u)
    \right\}.
\]
This is the usual generalization error associated with replacing the population functional by its empirical counterpart. A direct global empirical-process bound typically yields a slow rate of order \(n^{-1/2}\), up to logarithmic and spectral factors. A localized analysis can improve this to the fast scale \(n^{-1}\) for the corresponding excess-risk term. Related fast-rate analyses have been developed in~\cite{lei2025solving,lu2022machine,zhou2025weak}, either by enforcing the boundary condition exactly, which may be restrictive for practical neural-network hypothesis classes, or by working on boundaryless manifolds so that the trace term is absent. In contrast, the present spectral boundary formulation keeps the boundary contribution explicitly in the loss. After adding the spectral tail \(\tau_K(\mathcal F)\), the loss is aligned with the elliptic stability norm, which provides the local structure needed for fast-rate analysis: on suitable localized classes, the variance of the empirical loss can be controlled by the population error. Thus, the spectral boundary loss is not only numerically motivated, but also leads to an error functional better suited to sharp statistical estimates. A complete localized fast-rate analysis for the present boundary spectral functional is left for future work.
\end{remark}

\section{Experimental}\label{Section: Experimental}
The PICNN algorithm is implemented using Python scripts powered by the PyTorch framework. The scripts can be downloaded from \url{https://github.com/hanfei27/PINNOnManifolds}. We consider two representative manifolds: the upper hemisphere $\mathcal{M}_{\mathbb{S}}$ (equatorial boundary) and the upper half-torus $\mathcal{M}_{\mathbb{T}}$ (two circular boundaries at $z=0$). On each, we define a subdomain $\mathcal M^d$ ($d=2$) with boundary $\partial\mathcal M^d$. To evaluate the high-order boundary penalty in~\eqref{equation: modified_pinn_loss} efficiently, we employ Fast Fourier Transform (FFT) to implement the spectral estimation. This approach allows for the practical computation of fractional Sobolev norms by projecting the boundary residuals onto the frequency domain, enabling a computationally efficient realization of the PICNN loss. 

For subdomains on $\mathbb{S}^2$ and $\mathbb{T}^2$ with simple topologies, the boundary $\partial\mathcal{M}^d$ is a closed curve isometric to a circle $S^1_L$ of length $L$. Let $e(\tau) = (u_\theta - g)(\gamma(\tau))$ be the boundary error under arclength parametrization $\gamma: [0, L) \to \partial\mathcal{M}^d$. The Sobolev norm is characterized by the spectral decay:
\begin{equation*}
    \|e\|_{\mathcal{H}^{2s-\frac{1}{2}}(\partial\mathcal{M}^d)}^2 \asymp \sum_{k \in \mathbb{Z}} (1 + \lambda_k)^{2s-\frac{1}{2}} |\widehat{e}_k|^2, \quad \lambda_k = \left(\tfrac{2\pi k}{L}\right)^2,
\end{equation*}
where $\widehat{e}_k$ are the Fourier coefficients. In practice, we sample the error at $M$ equidistant points and compute $\widehat{e}_k$ via FFT. The norm is approximated by the truncated sum:
\begin{equation*}
    \mathcal{L}_{\mathrm{bnd}}^{\partial\mathcal{M}^d}(\theta) \approx \sum_{|k| \le K} (1 + \lambda_k)^{2s-\frac{1}{2}} |\widehat{e}_k|^2 = e_\theta^\top A_{\partial\mathcal{M}^d} e_\theta,
\end{equation*}
where $e_\theta \in \mathbb{R}^M$ is the discrete error vector and $A_{\partial\mathcal{M}^d}$ is a quadratic form matrix constructed from the FFT and spectral weights. This operation has a negligible cost of $\mathcal{O}(M \log M)$, allowing the boundary penalty to be efficiently updated at every training iteration.
\begin{algorithm}[!ht]
\caption{PICNN on Manifolds}
\label{alg:st_pinn}

\KwIn{Data $f, g$; samples $S^{\mathrm{int}}=\{x_i\}_{i=1}^n \subset \mathcal{M}^d$ and $S^{\mathrm{bnd}}=\{y_j\}_{j=1}^m \subset \partial \mathcal{M}^d$; parameters $N_{\mathrm{ep}}, s, K, \lambda_{\mathrm{bnd}}, \eta$; neural network $u_\theta$}
\KwOut{Optimized parameters $\theta^*$}

Initialize network parameters $\theta$\;
Precompute spectral weights $w_k = (1+\lambda_k)^{2s-\frac12}$ for $|k|\le K$\;
Set $\theta^* \leftarrow \theta$\;

\For{$t = 1, \ldots, N_{\mathrm{ep}}$}{
    \tcp{Interior loss}
    Compute
    \[
    \mathcal{L}_{\mathrm{phys}}(\theta)
    = \frac{1}{n}\sum_{i=1}^n \left|\mathcal{L}u_\theta(x_i)-f(x_i)\right|^2
    \]
    
    \tcp{Boundary loss}
    Set $\mathbf e = \{u_\theta(y_j)-g(y_j)\}_{j=1}^m$\;
    Compute $\widehat{\mathbf e}=\mathrm{FFT}(\mathbf e)$\;
    Compute
    \[
    \mathcal{L}_{\mathrm{bnd}}(\theta)
    = \sum_{|k|\le K} w_k |\widehat e_k|^2
    \]
    
    \tcp{Total loss}
    Compute
    \[
    \mathcal{L}_{\mathrm{total}}(\theta)
    = \mathcal{L}_{\mathrm{phys}}(\theta)
    + \lambda_{\mathrm{bnd}} \mathcal{L}_{\mathrm{bnd}}(\theta)
    \]
    
    Update
    \[
    \theta \leftarrow \theta - \eta \nabla_\theta \mathcal{L}_{\mathrm{total}}(\theta)
    \]
    
    \If{$\mathcal{L}_{\mathrm{total}}(\theta) < \mathcal{L}_{\mathrm{total}}(\theta^*)$}{
        $\theta^* \leftarrow \theta$\;
    }
}

\Return{$\theta^*$}\;
\end{algorithm}
On both domains, we solve the variable-coefficient elliptic equation:
\begin{equation}\label{eq:exp1_pde_both}
-\operatorname{div}_{\mathcal M}\!\big((2+z)\nabla_{\mathcal M}u\big)+u=f,
\qquad x\in\mathcal M,
\end{equation}
with the exact solution $u^*(x,y,z)=xyz$ and Dirichlet boundary data $g:=u^*|_{\partial\mathcal M}$. The source term $f$ is derived analytically to satisfy \eqref{eq:exp1_pde_both}.

\subsection{Loss Function and Boundary Penalty}
We train the network $u_\theta$ by minimizing the total empirical loss $\mathcal L(\theta) := \mathcal L_{\mathrm{phys}}(\theta) + \lambda_{\mathrm{bnd}} \mathcal L_{\mathrm{bnd}}(\theta),$ where $\mathcal L_{\mathrm{phys}}$ is the mean-squared interior residual. The boundary term $\mathcal L_{\mathrm{bnd}}$ approximates the $\|u_\theta - g\|_{\mathcal H^{3/2}}^2$ norm using discrete spectral estimators. For the hemisphere, this is computed via a single 1D FFT. For the torus, the boundary loss is defined as the sum of spectral penalties on its two disjoint components:
\begin{equation*}
\mathcal L_{\mathrm{bnd}}^{\partial\mathcal M_{\mathbb T}}(\theta) := \|u_\theta - g\|_{\mathcal H^{3/2}(\Gamma_+)}^2 + \|u_\theta - g\|_{\mathcal H^{3/2}(\Gamma_-)}^2,
\end{equation*}
where each term is evaluated independently using the FFT-based approach described in~\autoref{alg:st_pinn}. This formulation ensures that high-order boundary stability is maintained across disjoint boundary components.

\subsection{Network Architecture and Smooth Activation}
We employ a lightweight 1D-CNN + Multilayer Perceptron (MLP) architecture. The input $x\in\mathbb R^3$ is reshaped into a one-dimensional sequence and passed through $L$ convolutional layers with channel width $c$, followed by average pooling and flattening. The resulting features are then fed into a two-layer MLP to produce a scalar output. Although our approximation theory is developed for $\mathrm{ReLU}$--$\mathrm{ReQU}$ networks, preliminary experiments indicate that $\mathrm{ReLU}$ is not well suited for PINNs, since its nondifferentiability at zero makes high-order differentiation unstable and may lead to training failure. For this reason, we replace $\mathrm{ReLU}$ by the Gaussian Error Linear Unit (GeLU)~\cite{hendrycks2016gaussian}. Empirically, this yields substantially improved stability and convergence. More precisely, we use the $\mathrm{GeLU}^2(t)$ in the MLP layers to enhance regularity.

\subsection{Optimization and Evaluation}
Parameters are optimized using Adam with a StepLR scheduler. We denote the number of interior samples by $N$, boundary samples by $M$, and independent test samples by $N_{\mathrm{test}}$. At test time, we report relative $L^2$ and $H^2$ errors:
\begin{equation*}
    \mathrm{Rel}\ L^2 :=\frac{\|u_\theta-u^*\|_{L^2(\mathcal M)}}{\|u^*\|_{L^2(\mathcal M)}}, \qquad
    \mathrm{Rel}\ H^{2} :=\frac{\|\Delta_{\mathcal M}u_\theta-\Delta_{\mathcal M}u^*\|_{L^2(\mathcal M)}}{\|\Delta_{\mathcal M}u^*\|_{L^2(\mathcal M)}},
\end{equation*}
approximated by discrete averages over the test set. 

\subsection{Results and Discussion}\label{Subsubsection: Results and Discussion}
We investigate how different boundary losses affect error convergence as the number of interior samples $N$ increases. Fixing the network architecture and optimization strategy, we vary $N$ and compare the standard $L^2$ penalty against our proposed Sobolev penalty:
\begin{equation*}
    \mathcal L_{\mathrm{bnd}}^{L^2}(\theta)=\|u_\theta-g\|_{L^2(\partial\mathcal M)}^2 \quad \text{vs.} \quad
    \mathcal L_{\mathrm{bnd}}^{H^{3/2}}(\theta)=\|u_\theta-g\|_{\mathcal H^{3/2}(\partial\mathcal M)}^2.
\end{equation*}
We evaluate performance using both $\mathrm{Rel}\ L^2$ and $\mathrm{Rel}\ H^{2}$ metrics.

Both manifolds are tested under the same setup. We vary the interior sample size over $N\in\{128,256,512,1024,2048,4096\}$, while fixing $M=256$ and $N_{\mathrm{test}}=5120$. For each $N$, we perform 10 independent trials and report the mean and standard deviation. The network is a 1D-CNN + MLP with $L=3$ convolutional layers of 56 channels and an MLP head of widths $(24,8)$. Training uses Adam with initial learning rate $\eta=10^{-3}$ for $T=200$ epochs, with decay factor $\gamma=0.5$ every $\tau=50$ epochs. We report the test errors $\mathrm{Rel}\,L^2$ and $\mathrm{Rel}\,H^{2s}$; see~\autoref{tab:exp1_relL2}--\autoref{tab:exp1_relH2s}.
\begin{table}[H]
\centering
\small
\setlength{\tabcolsep}{4pt}
\renewcommand{\arraystretch}{1.15}
\caption{Test $\mathrm{Rel}\,L^2$ errors under different boundary losses and manifold types.}
\label{tab:exp1_relL2}
\resizebox{\linewidth}{!}{%
\begin{tabular}{llcccccc}
\hline
Manifold $\mathcal M$ & Boundary loss $\mathcal L_{\mathrm{bnd}}$ 
& $N=128$ & $256$ & $512$ & $1024$ & $2048$ & $4096$ \\
\hline
$\mathcal M_{\mathbb S}$ & $H^{2s-\frac12}$ 
& $0.013400\pm0.005801$ & $0.015420\pm0.006197$ & $0.007014\pm0.002749$
& $0.003877\pm0.001847$ & $0.003638\pm0.001473$ & $0.003826\pm0.001525$ \\
$\mathcal M_{\mathbb S}$ & $L^2$
& $0.103385\pm0.019942$ & $0.056990\pm0.030944$ & $0.032291\pm0.009287$
& $0.016695\pm0.009112$ & $0.013323\pm0.004489$ & $0.007756\pm0.003632$ \\
$\mathcal M_{\mathbb T}$ & $H^{2s-\frac12}$  
& $0.534112\pm0.344148$ & $0.040684\pm0.055469$ & $0.006771\pm0.003217$
& $0.003835\pm0.001095$ & $0.002760\pm0.000958$ & $0.002397\pm0.000866$ \\
$\mathcal M_{\mathbb T}$ & $L^2$
& $0.110458\pm0.102444$ & $0.038127\pm0.016059$ & $0.029580\pm0.019402$
& $0.011489\pm0.003508$ & $0.009491\pm0.002787$ & $0.008118\pm0.007250$ \\
\hline
\end{tabular}}
\end{table}
Following the results in \autoref{tab:exp1_relL2} and \autoref{tab:exp1_relH2s}, we visualize the convergence behavior in log--log coordinates. For each sample size $N$, we compute the mean error $\bar e_N$ and standard deviation across $n_{\mathrm{runs}}$ independent trials. The empirical convergence rate $\alpha$ is then determined via a linear least-squares fit $\log_2 \bar e_N \approx a - \alpha \log_2 N.$
\begin{table}[H]
\centering
\small
\setlength{\tabcolsep}{4pt}
\renewcommand{\arraystretch}{1.15}
\caption{Test $\mathrm{Rel}\,H^{2s}$ errors under different boundary losses and manifold types.}
\label{tab:exp1_relH2s}
\resizebox{\linewidth}{!}{%
\begin{tabular}{llcccccc}
\hline
Manifold $\mathcal M$ & Boundary loss $\mathcal L_{\mathrm{bnd}}$ 
& $N=128$ & $256$ & $512$ & $1024$ & $2048$ & $4096$ \\
\hline
$\mathcal M_{\mathbb S}$ & $H^{2s-\frac12}$ 
& $0.014300\pm0.006167$ & $0.011610\pm0.003901$ & $0.005242\pm0.001911$
& $0.002926\pm0.001071$ & $0.002434\pm0.000768$ & $0.001895\pm0.000413$ \\
$\mathcal M_{\mathbb S}$ & $L^2$
& $0.026414\pm0.005214$ & $0.014656\pm0.004128$ & $0.009338\pm0.002966$
& $0.005465\pm0.001717$ & $0.004558\pm0.000935$ & $0.002905\pm0.000930$ \\
$\mathcal M_{\mathbb T}$ & $H^{2s-\frac12}$ 
& $0.919617\pm0.547020$ & $0.171405\pm0.186194$ & $0.022856\pm0.007895$
& $0.013131\pm0.002706$ & $0.010021\pm0.001337$ & $0.009999\pm0.002021$ \\
$\mathcal M_{\mathbb T}$ & $L^2$
& $0.220294\pm0.161027$ & $0.053042\pm0.033999$ & $0.025099\pm0.019792$
& $0.010551\pm0.003485$ & $0.006743\pm0.002244$ & $0.005756\pm0.006204$ \\
\hline
\end{tabular}}
\end{table}
\autoref{fig:exp1_all_rates} shows the convergence results on the upper hemisphere $\mathcal M_{\mathbb S}$ (top row) and the upper half-torus $\mathcal M_{\mathbb T}$ (bottom row). The left column uses the Sobolev boundary penalty, while the right column uses the standard $L^2$ penalty. All panels are plotted on a log-log scale, with test error versus the number of interior samples $N$, and the fitted lines indicate the empirical convergence rates. The error bars represent the standard deviation over $n_{\mathrm{runs}}$ independent trials.
\begin{figure}[!ht]
\centering

\begin{subfigure}{0.48\linewidth}
  \centering
  \includegraphics[width=\linewidth]{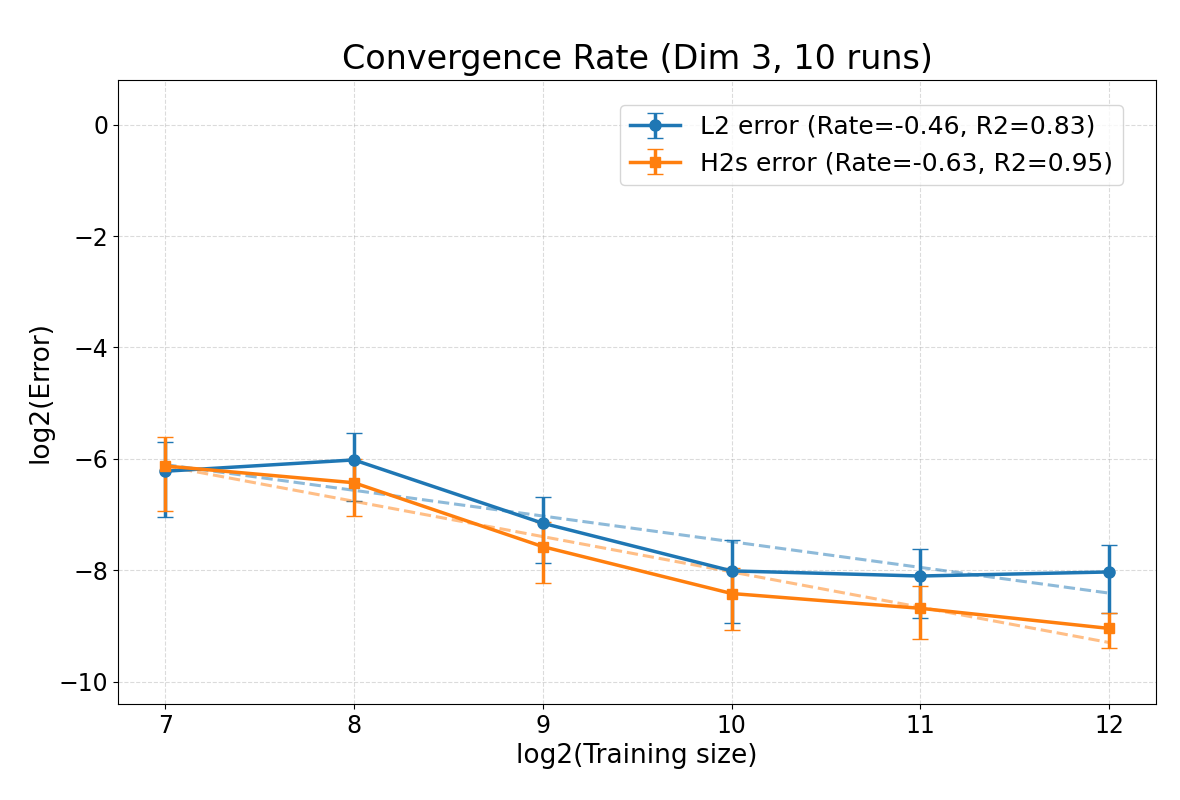}
  \caption{$\mathcal M_{\mathbb S}$ + $\mathcal L_{\mathrm{bnd}}^{H^{2s-\frac12}}$}
  \label{fig:exp1_sphere_sobolev_rate}
\end{subfigure}\hfill
\begin{subfigure}{0.48\linewidth}
  \centering
  \includegraphics[width=\linewidth]{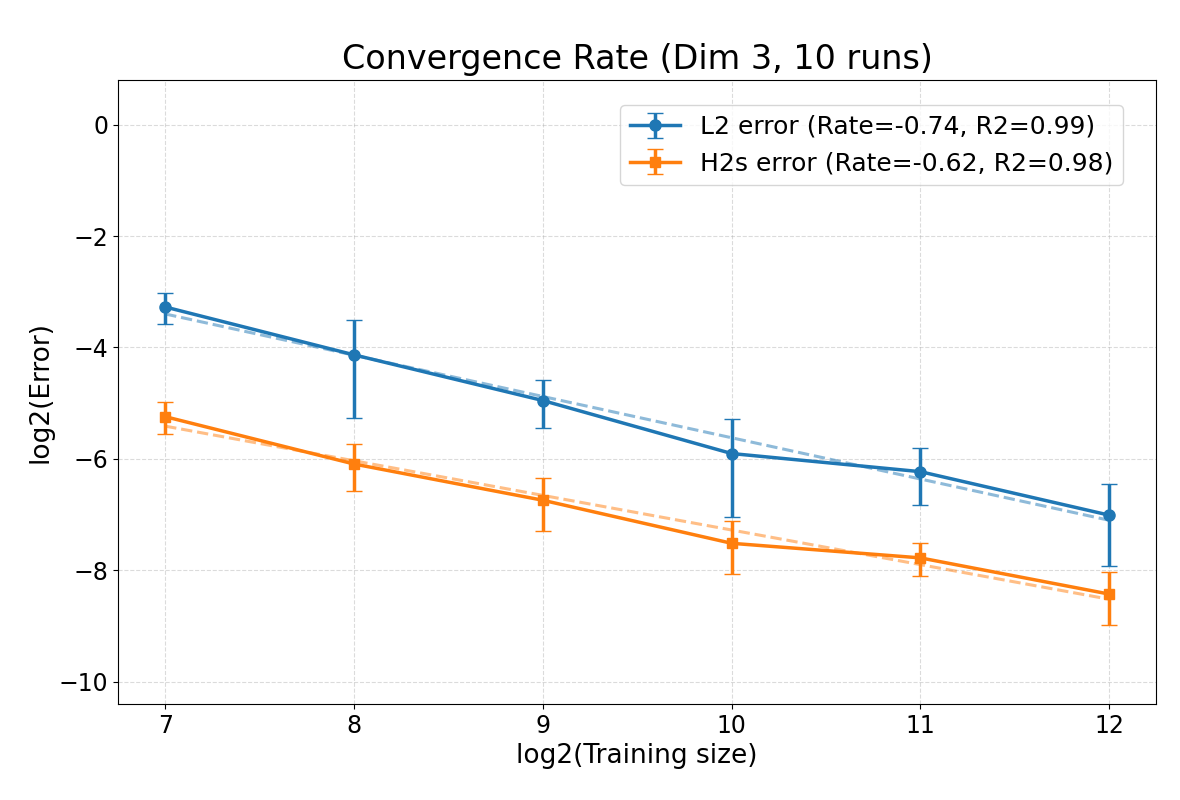}
  \caption{$\mathcal M_{\mathbb S}$ + $\mathcal L_{\mathrm{bnd}}^{L^2}$}
  \label{fig:exp1_sphere_l2_rate}
\end{subfigure}

\vspace{0.6em}

\begin{subfigure}{0.48\linewidth}
  \centering
  \includegraphics[width=\linewidth]{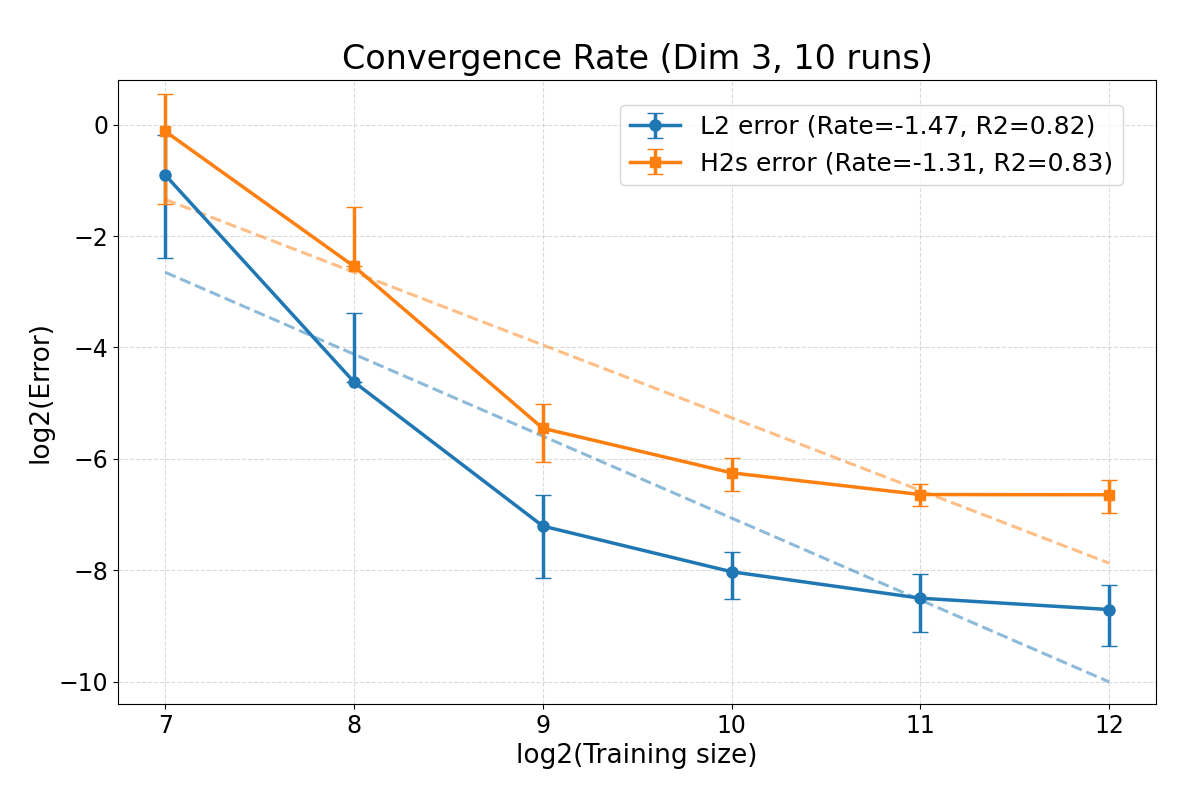}
  \caption{$\mathcal M_{\mathbb T}$ + $\mathcal L_{\mathrm{bnd}}^{H^{2s-\frac12}}$}
  \label{fig:exp1_torus_sobolev_rate}
\end{subfigure}\hfill
\begin{subfigure}{0.48\linewidth}
  \centering
  \includegraphics[width=\linewidth]{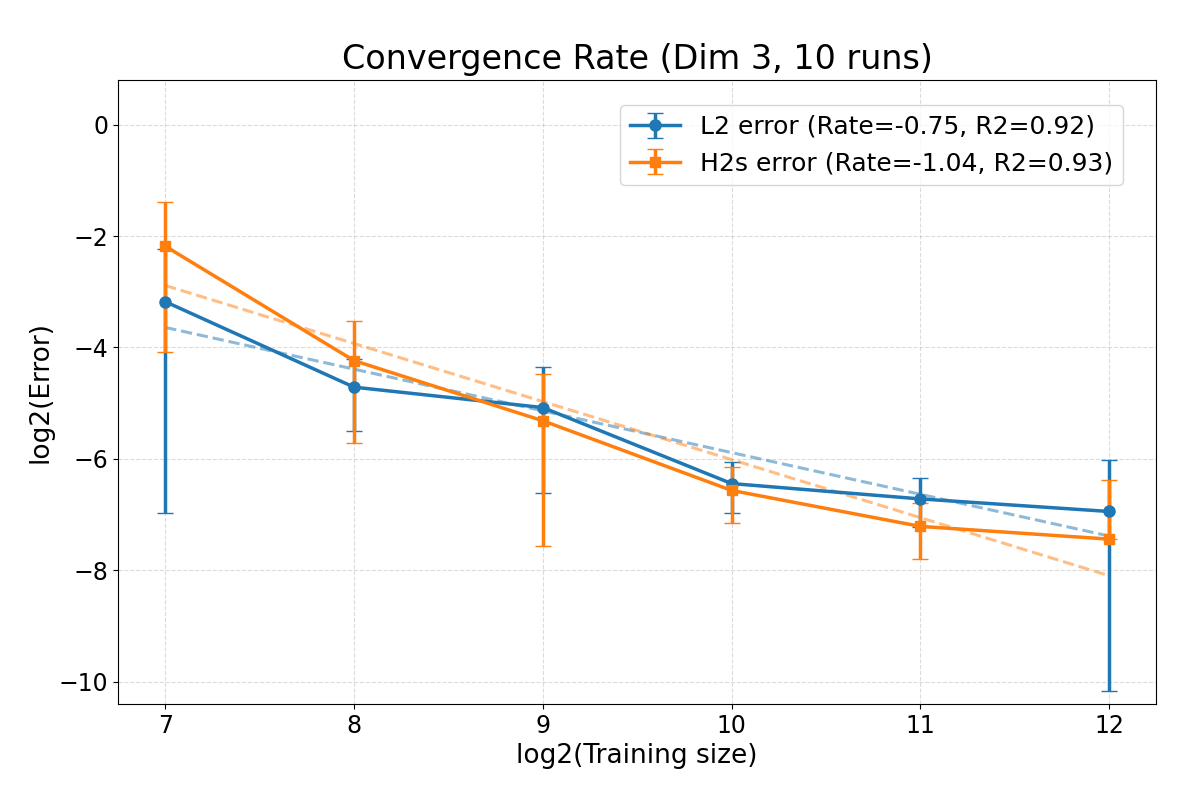}
  \caption{$\mathcal M_{\mathbb T}$ + $\mathcal L_{\mathrm{bnd}}^{L^2}$}
  \label{fig:exp1_torus_l2_rate}
\end{subfigure}

\caption{Empirical convergence curves (log--log) and fitted slopes on two manifolds $\mathcal M_{\mathbb S}$ and $\mathcal M_{\mathbb T}$, comparing the boundary losses $\mathcal L_{\mathrm{bnd}}^{H^{2s-\frac12}}$ and $\mathcal L_{\mathrm{bnd}}^{L^2}$.}
\label{fig:exp1_all_rates}
\end{figure}
As shown in \autoref{fig:exp1_all_rates}, the Sobolev penalty $\mathcal L_{\mathrm{bnd}}^{H^{3/2}}$ consistently outperforms the standard $L^2$ penalty on both manifolds. This advantage is evident in both $\mathrm{Rel}\,L^2$ and $\mathrm{Rel}\,H^{2}$ metrics. Intuitively, the Sobolev penalty goes beyond simple value matching: it actively suppresses high-frequency oscillations near the boundary. Since the PDE involves second-order derivatives, this added regularity stabilizes training and leads to better final accuracy.

The log-log error curves do not show a perfect power-law behavior. We suspect this is due to optimization dynamics rather than network limitations. Training PINNs with high-order derivatives is sensitive to hyperparameters, so under a fixed training budget, experiments with different $N$ may not reach the same optimization level. Consequently, the fitted rates reflect a mix of statistical scaling and optimization noise. Future work could focus on designing more robust strategies to stabilize the Sobolev penalty against high-frequency oscillations.

To analyze the spatial error distribution, we visualize pointwise errors for $N=4096$ in \autoref{fig:exp1_vis_4096}. We compare the function error $|u_\theta-u|$ and Laplacian error $|\Delta_{\mathcal M}u_\theta-\Delta_{\mathcal M}u|$ across all four settings. 

\begin{figure}[!ht]
\centering
\begin{subfigure}{0.48\linewidth}
  \centering
  \includegraphics[width=\linewidth,height=0.28\textheight,keepaspectratio]{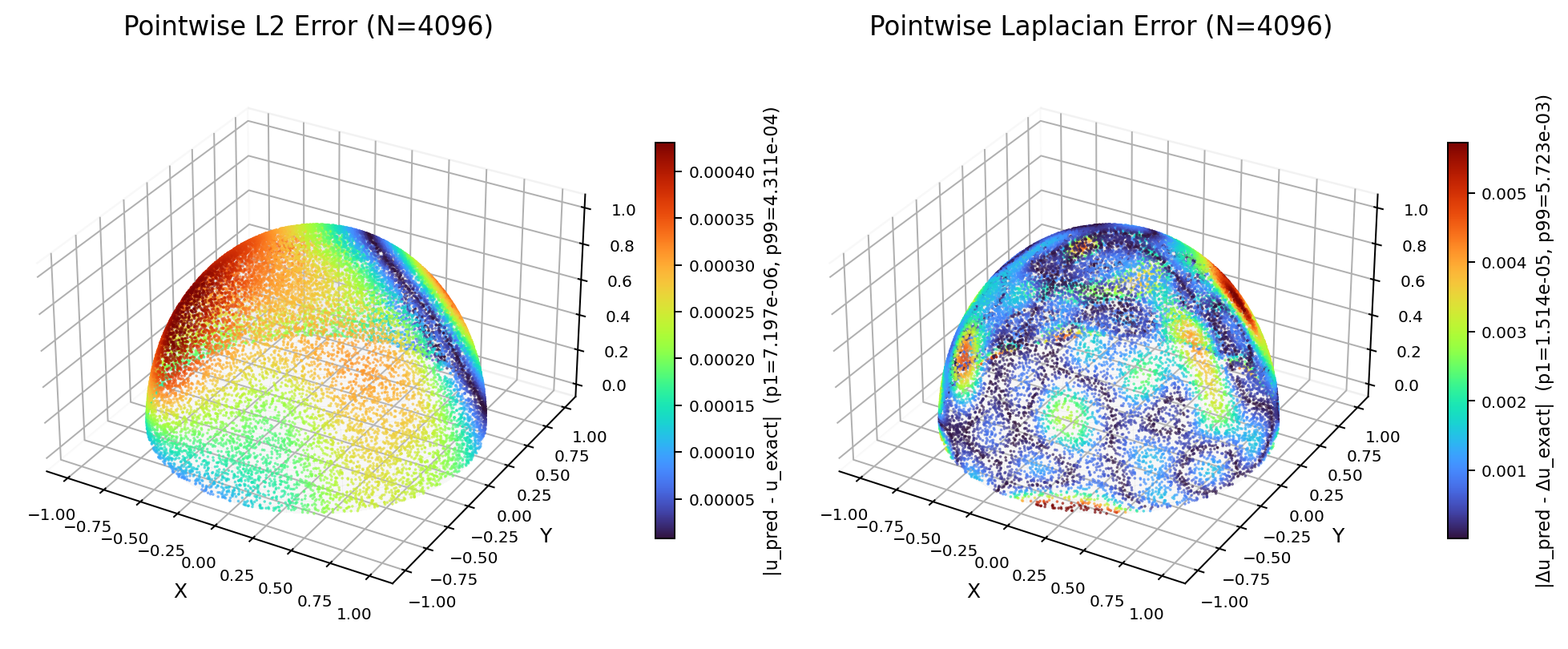}
  \caption{$\mathcal M_{\mathbb S}+\mathcal L_{\mathrm{bnd}}^{H^{2s-\frac12}}$}
  \label{fig:exp1_sphere_sobolev_vis}
\end{subfigure}\hfill
\begin{subfigure}{0.48\linewidth}
  \centering
  \includegraphics[width=\linewidth,height=0.28\textheight,keepaspectratio]{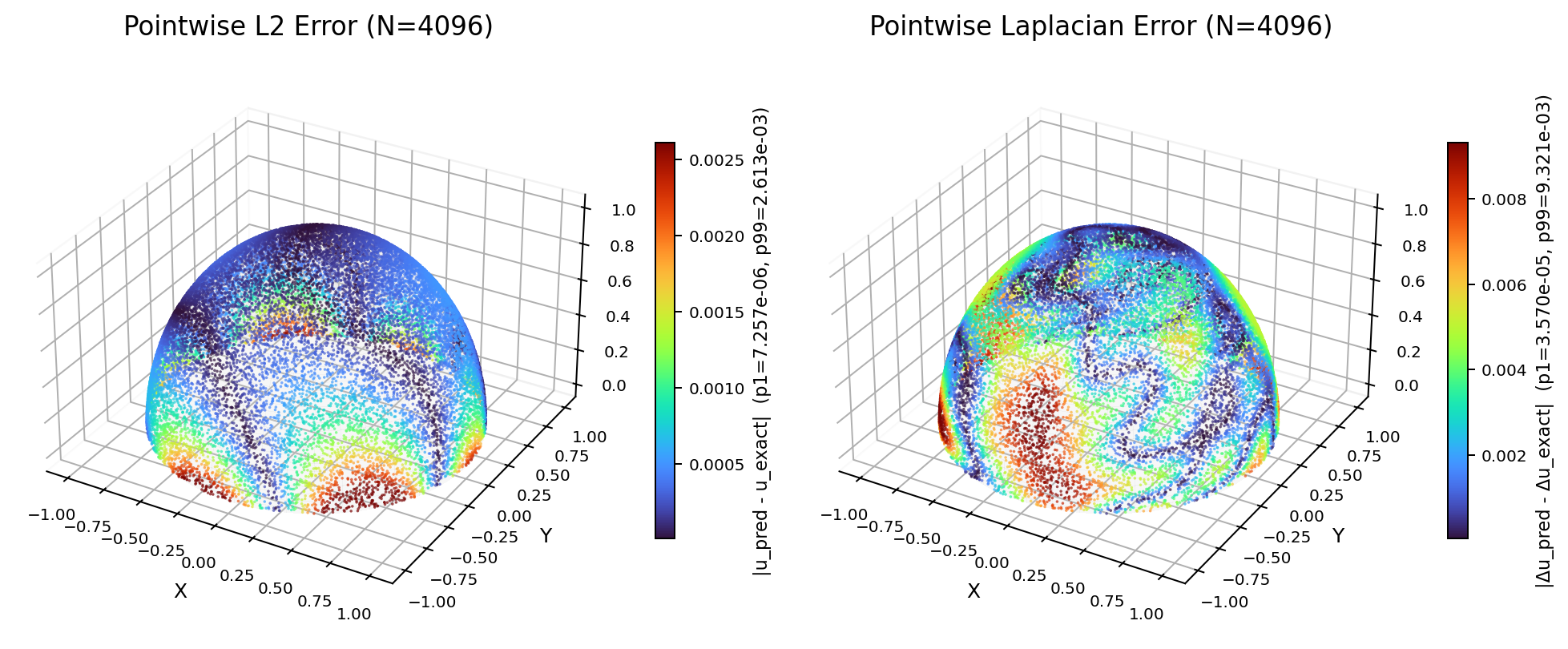}
  \caption{$\mathcal M_{\mathbb S}+\mathcal L_{\mathrm{bnd}}^{L^2}$}
  \label{fig:exp1_sphere_l2_vis}
\end{subfigure}

\vspace{0.6em}

\begin{subfigure}[t]{0.48\linewidth}
  \centering
  \includegraphics[width=\linewidth,height=0.28\textheight,keepaspectratio]{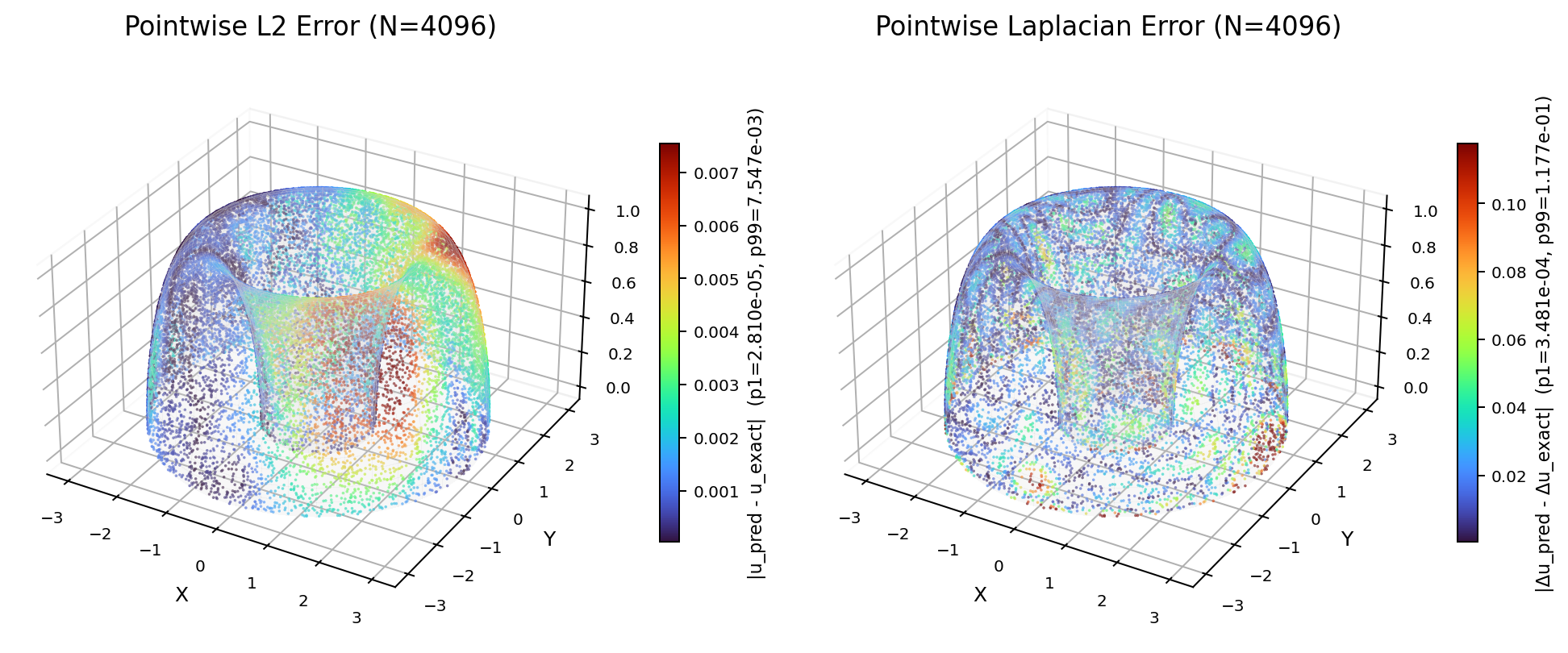}
  \caption{$\mathcal M_{\mathbb T}+\mathcal L_{\mathrm{bnd}}^{H^{2s-\frac12}}$}
  \label{fig:exp1_torus_sobolev_vis}
\end{subfigure}\hfill
\begin{subfigure}[t]{0.48\linewidth}
  \centering
  \includegraphics[width=\linewidth,height=0.28\textheight,keepaspectratio]{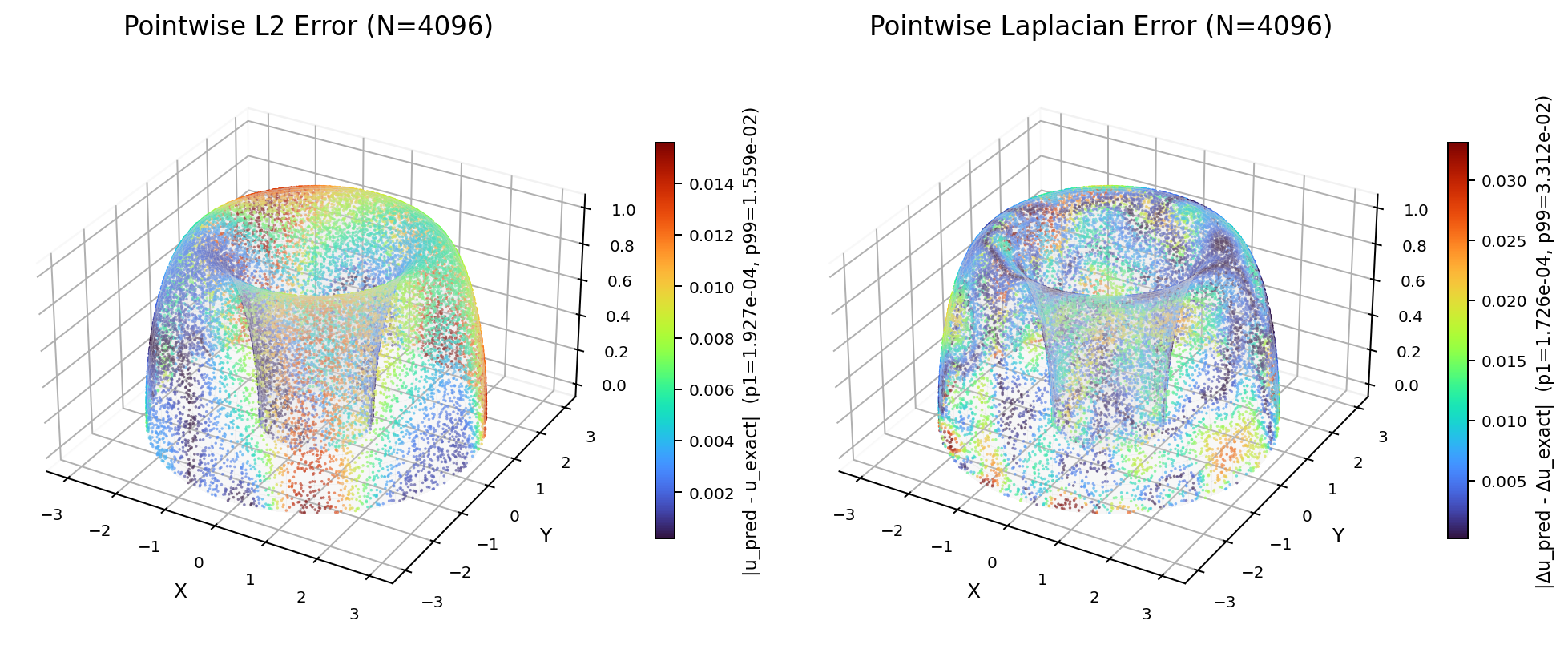}
  \caption{$\mathcal M_{\mathbb T}+\mathcal L_{\mathrm{bnd}}^{L^2}$}
  \label{fig:exp1_torus_l2_vis}
\end{subfigure}
\caption{Pointwise error visualization at $N=4096$: left/right subplots correspond to the spatial distributions of $|u_\theta-u|$ and $|\Delta_{\mathcal M}u_\theta-\Delta_{\mathcal M}u|$, comparing four combinations of boundary penalty and manifold type.}
\label{fig:exp1_vis_4096}
\end{figure}
To study the optimization behavior, we track the test $\mathrm{Rel}\,L^2$ and $\mathrm{Rel}\,H^2$ errors over epochs on both $\mathcal M_{\mathbb S}$ and $\mathcal M_{\mathbb T}$. With all other hyperparameters fixed, we compare the Sobolev penalty with the standard $L^2$ penalty. Empirically, the Sobolev penalty $\mathcal L_{\mathrm{bnd}}^{H^{3/2}}$ yields faster error decay in the early stage and smoother convergence trajectories. \autoref{fig:epoch_6panels} shows the corresponding curves for $N\in\{512,1024,2048\}$. This improvement can be attributed to the spectral nature of the Sobolev norm, which places greater weight on high-frequency modes and hence suppresses boundary oscillations more effectively. Moreover, because this norm matches the natural trace space of the PDE, it provides a more coherent scaling of the error components, leading to more efficient gradient descent and earlier convergence under the same training budget.
\begin{figure}[!ht]
\centering

\begin{subfigure}[t]{0.32\linewidth}
  \centering
  \includegraphics[width=\linewidth,height=0.26\textheight,keepaspectratio]{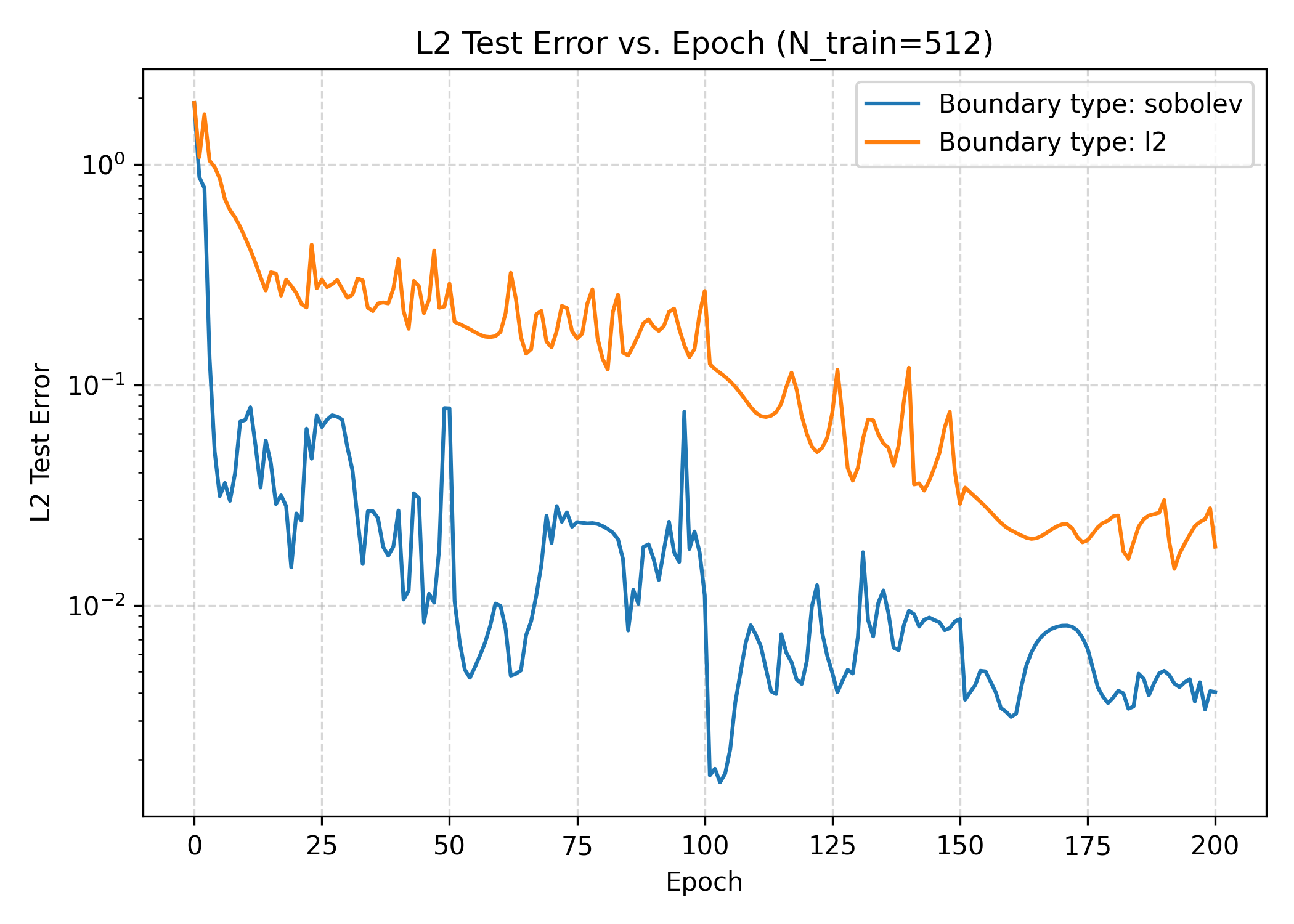}
  \caption{$\mathcal M_{\mathbb S},\,N=512$}
  \label{fig:epoch_sphere_n512}
\end{subfigure}\hfill
\begin{subfigure}[t]{0.32\linewidth}
  \centering
  \includegraphics[width=\linewidth,height=0.26\textheight,keepaspectratio]{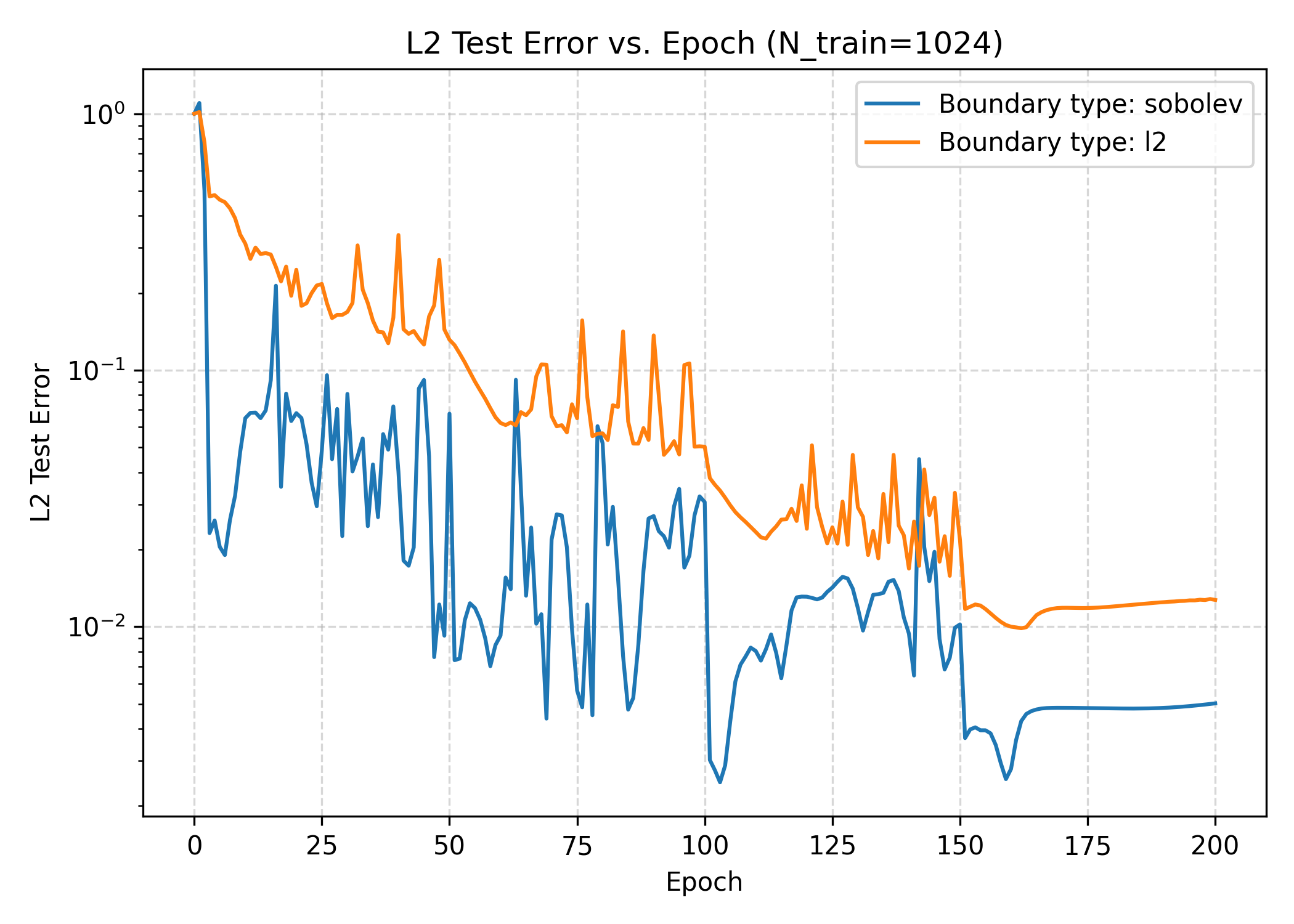}
  \caption{$\mathcal M_{\mathbb S},\,N=1024$}
  \label{fig:epoch_sphere_n1024}
\end{subfigure}\hfill
\begin{subfigure}[t]{0.32\linewidth}
  \centering
  \includegraphics[width=\linewidth,height=0.26\textheight,keepaspectratio]{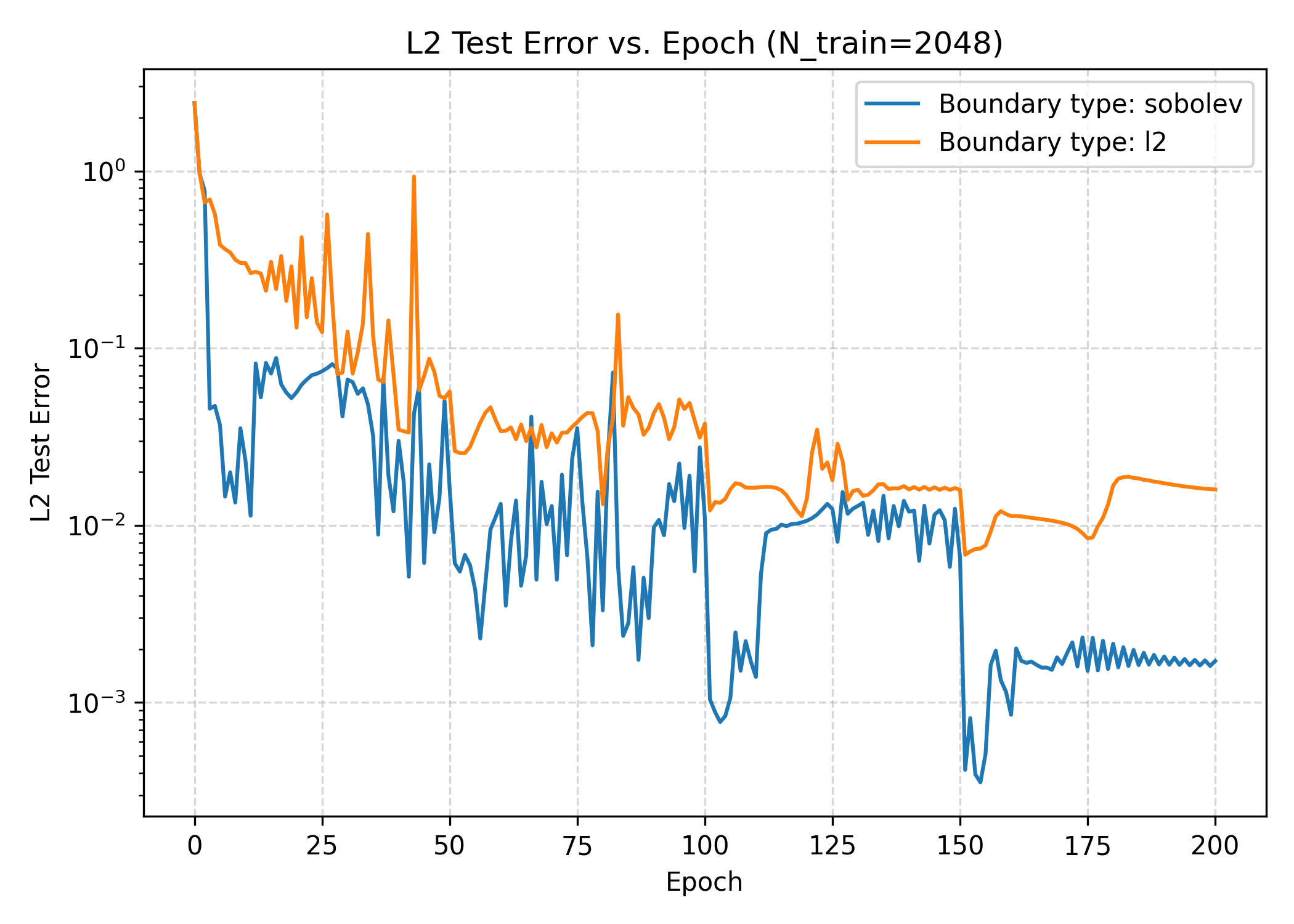}
  \caption{$\mathcal M_{\mathbb S},\,N=2048$}
  \label{fig:epoch_sphere_n2048}
\end{subfigure}

\vspace{0.6em}

\begin{subfigure}[t]{0.32\linewidth}
  \centering
  \includegraphics[width=\linewidth,height=0.26\textheight,keepaspectratio]{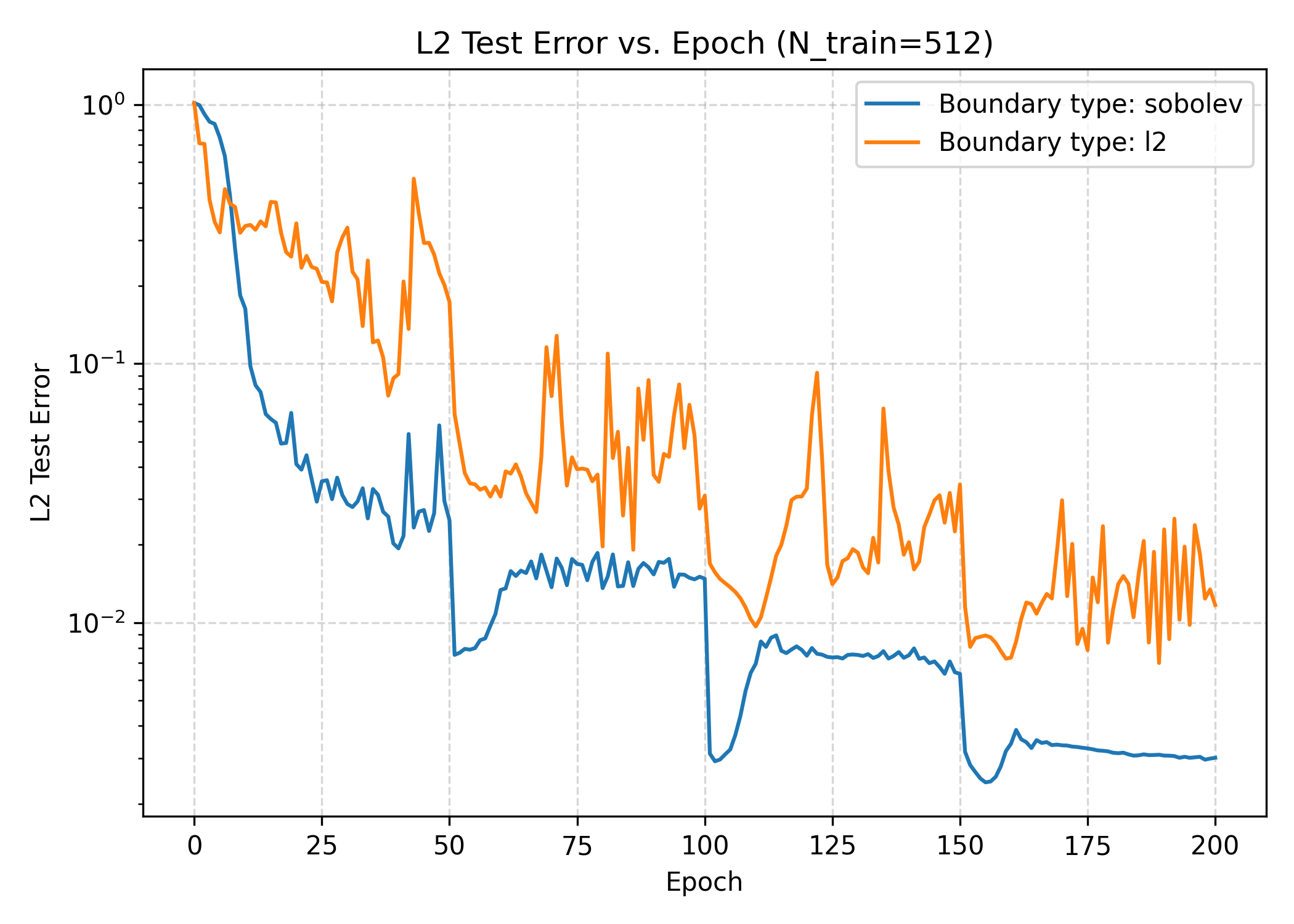}
  \caption{$\mathcal M_{\mathbb T},\,N=512$}
  \label{fig:epoch_torus_n512}
\end{subfigure}\hfill
\begin{subfigure}[t]{0.32\linewidth}
  \centering
  \includegraphics[width=\linewidth,height=0.26\textheight,keepaspectratio]{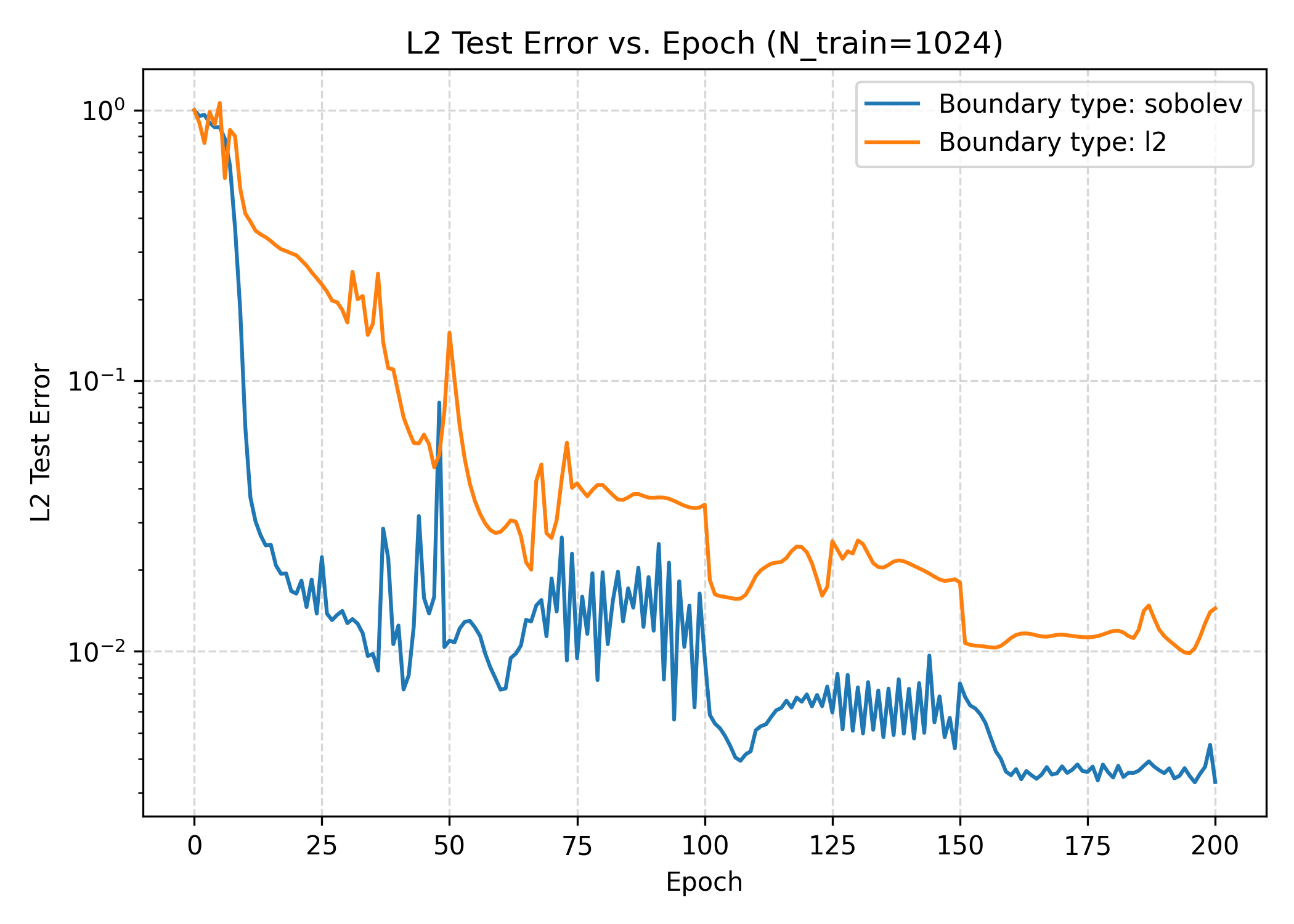}
  \caption{$\mathcal M_{\mathbb T},\,N=1024$}
  \label{fig:epoch_torus_n1024}
\end{subfigure}\hfill
\begin{subfigure}[t]{0.32\linewidth}
  \centering
  \includegraphics[width=\linewidth,height=0.26\textheight,keepaspectratio]{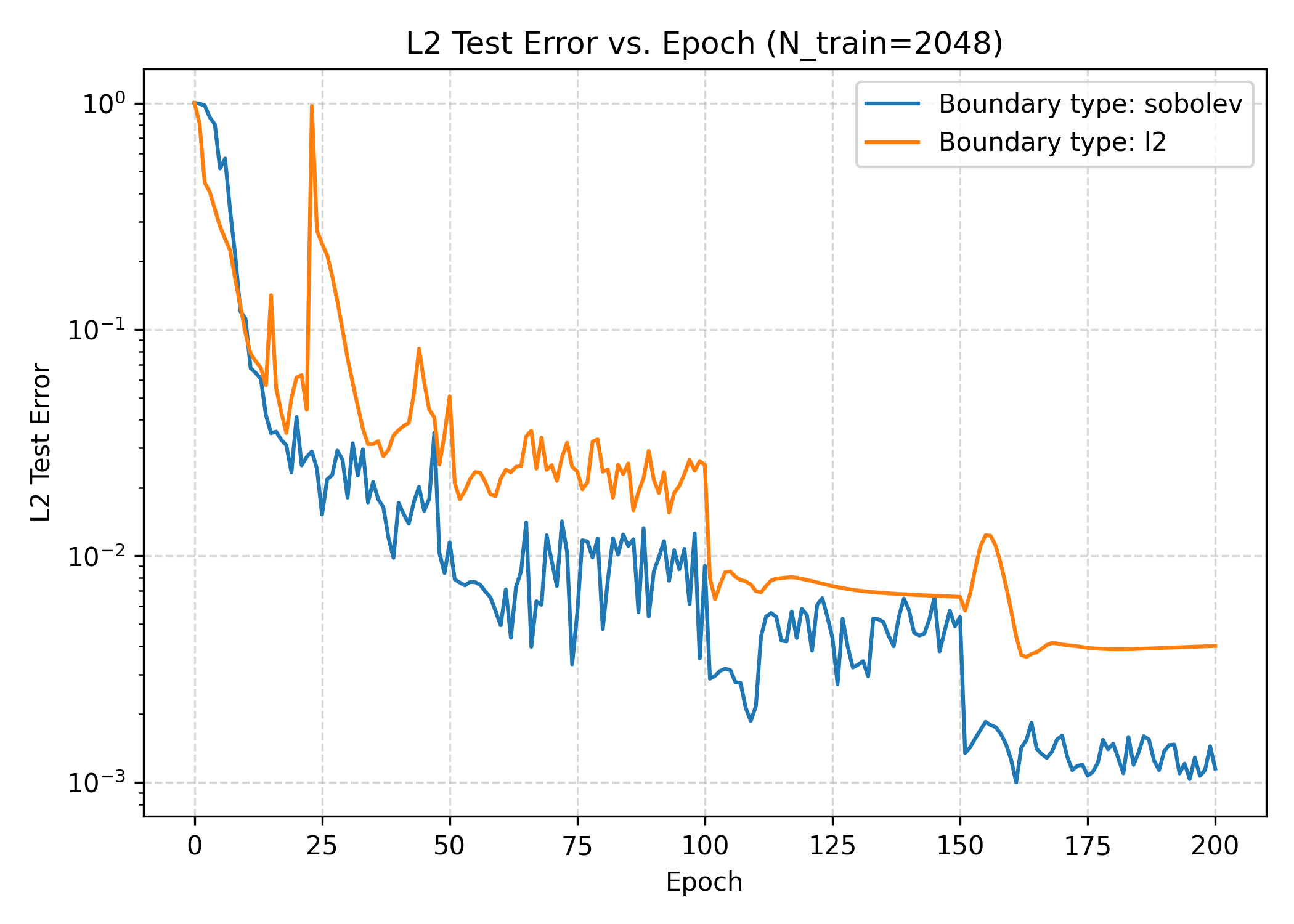}
  \caption{$\mathcal M_{\mathbb T},\,N=2048$}
  \label{fig:epoch_torus_n2048}
\end{subfigure}

\caption{Comparison of test $L^2$ error versus epoch under different interior sample sizes.}
\label{fig:epoch_6panels}
\end{figure}

We next assess whether increasing network width reduces error under a fixed sampling budget. On the upper torus $\mathcal M_{\mathbb T}$, we perform a channel-width sweep ($c \in \{2, \dots, 64\}$) for a 2-layer CNN. We fix the geometry, Sobolev boundary penalty ($\lambda=10$), and optimization scheme (Adam, 200 epochs). We report statistics over 5 runs using interior sizes $N\in\{16384, 32768\}$ and boundary size $M=2048$. As shown in \autoref{fig:feature_scaling_4panels}, test errors decrease as $c$ grows, confirming that expressive features help generalization. However, performance eventually saturates. This suggests that beyond a certain point, errors are dominated by data scarcity and optimization limits rather than capacity. Therefore, one must carefully balance model size against available data and compute.

\begin{figure}[!ht]
\centering

\begin{subfigure}[t]{0.44\linewidth}
  \centering
  \includegraphics[width=\linewidth,height=0.24\textheight,keepaspectratio]{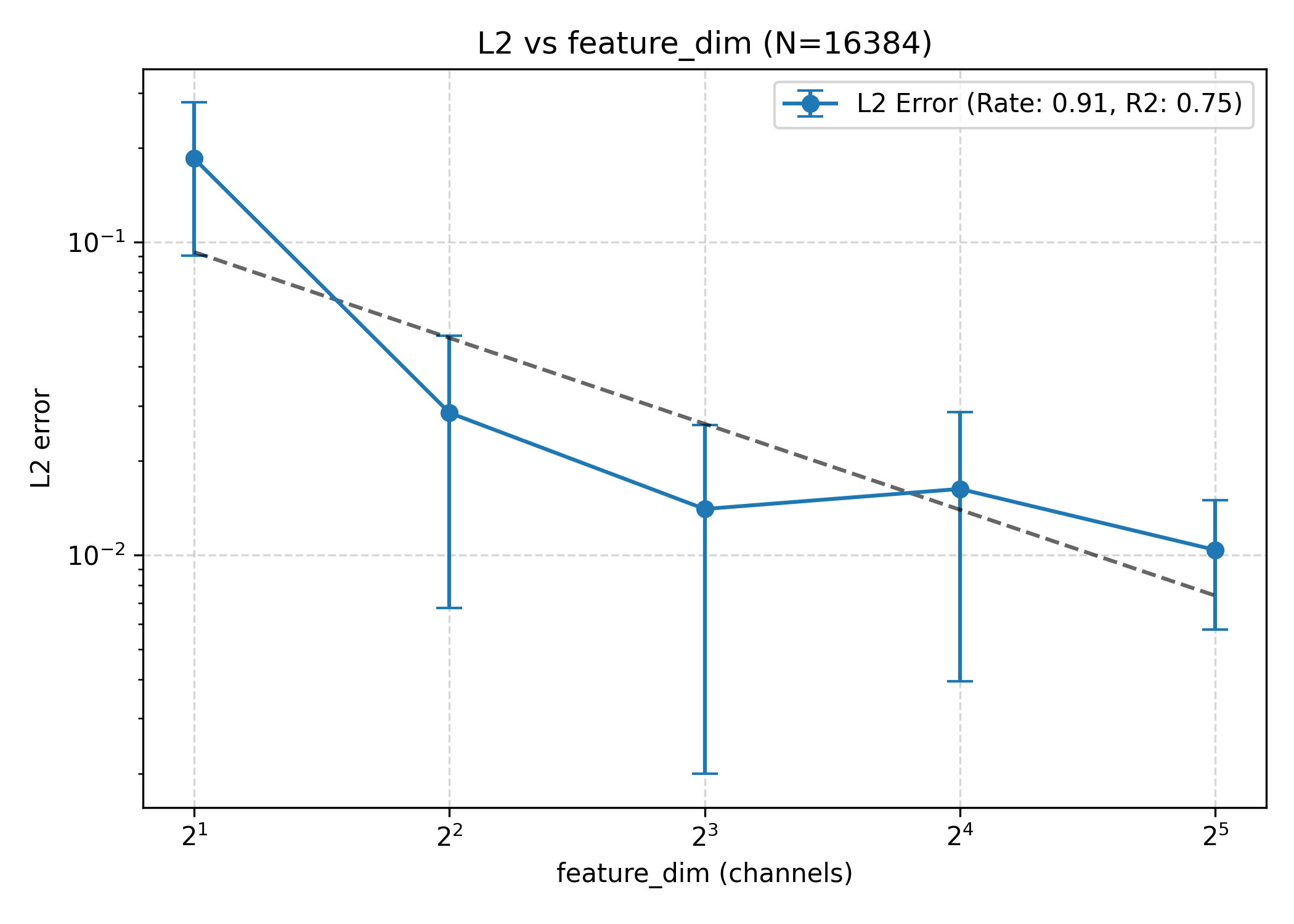}
  \caption{$\mathrm{Rel}\,L^2,\ N=16384$}
  \label{fig:feat_l2_N16384}
\end{subfigure}\hfill
\begin{subfigure}[t]{0.44\linewidth}
  \centering
  \includegraphics[width=\linewidth,height=0.24\textheight,keepaspectratio]{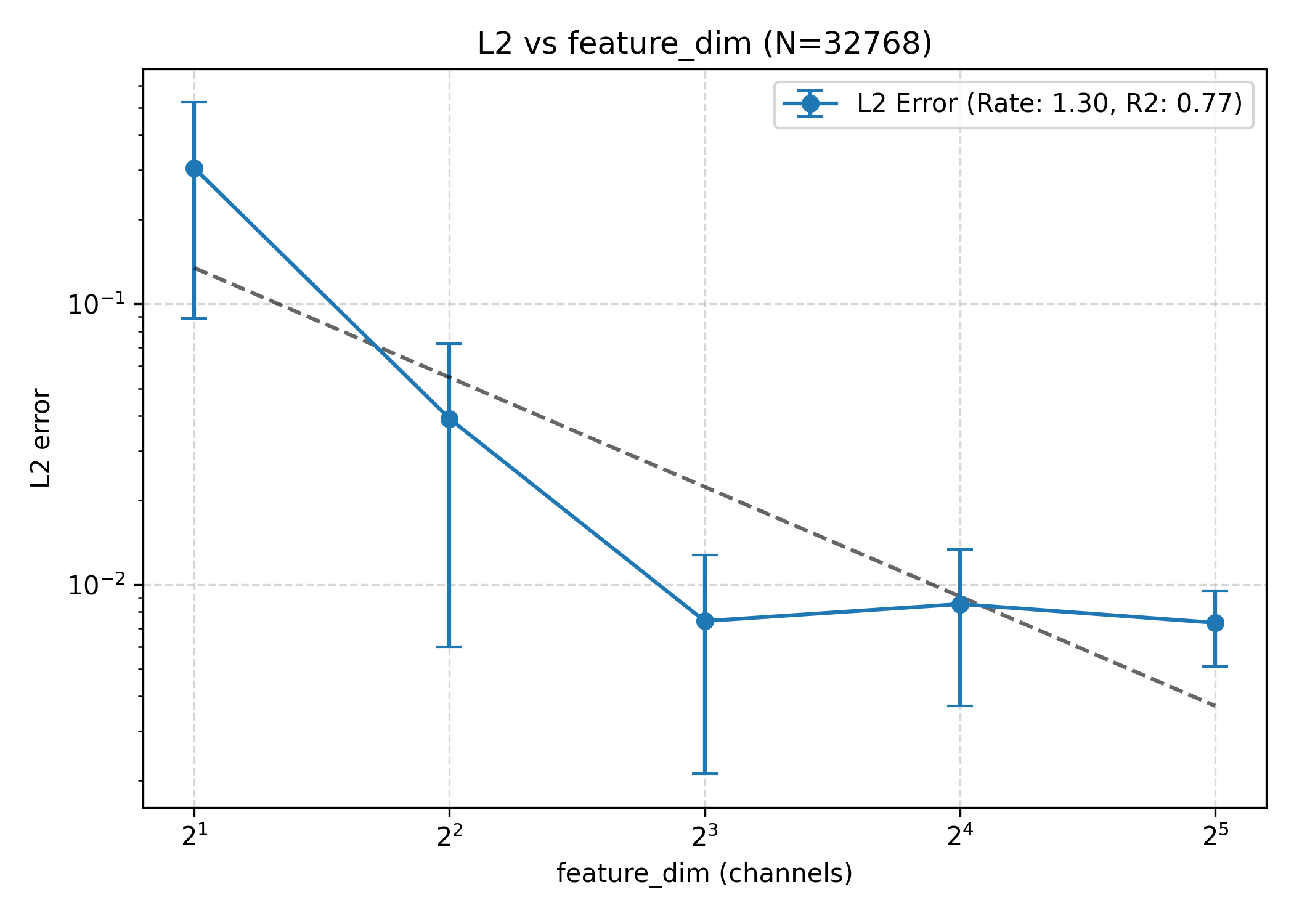}
  \caption{$\mathrm{Rel}\,L^2,\ N=32768$}
  \label{fig:feat_l2_N32768}
\end{subfigure}

\vspace{0.5em}

\begin{subfigure}[t]{0.44\linewidth}
  \centering
  \includegraphics[width=\linewidth,height=0.24\textheight,keepaspectratio]{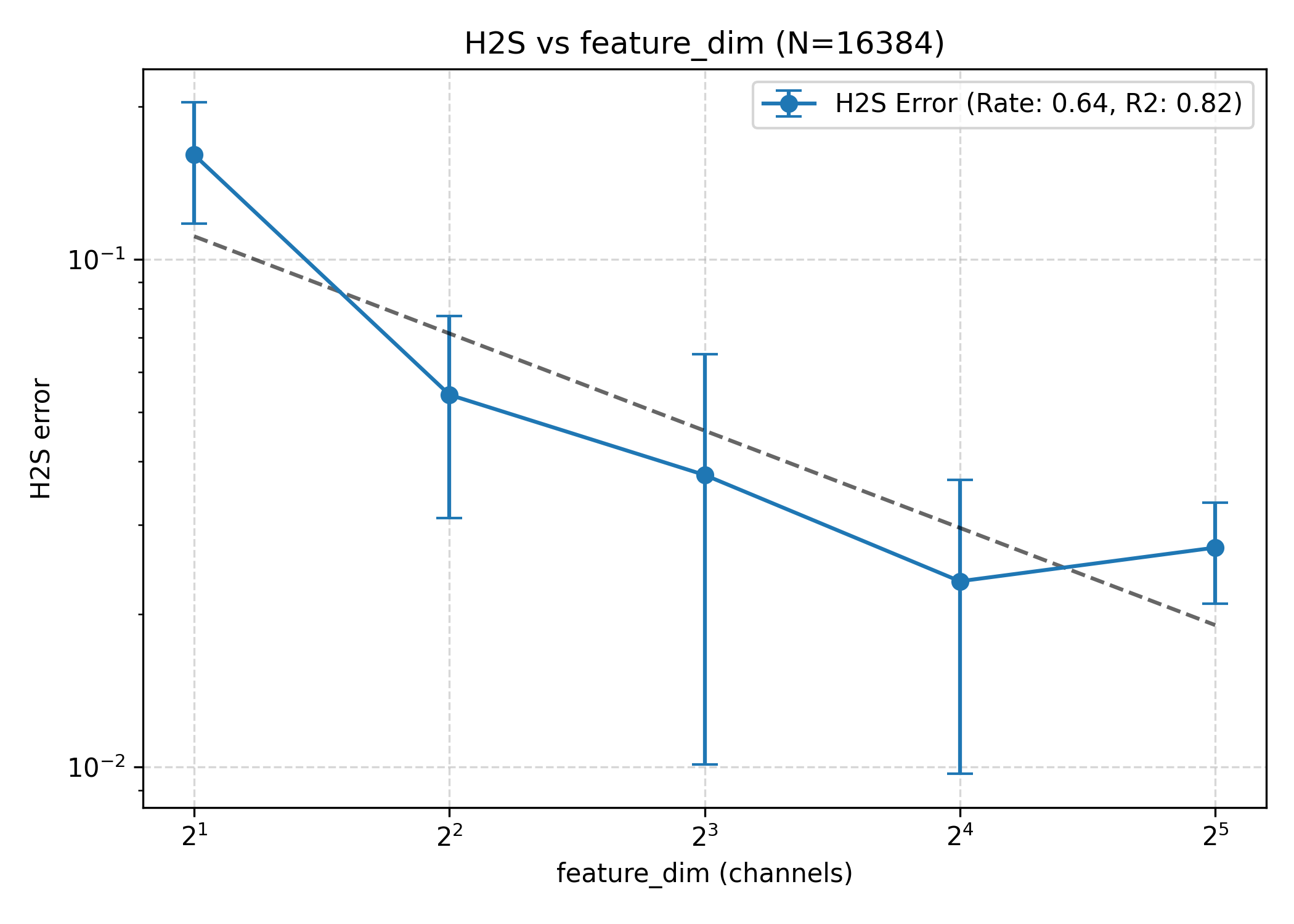}
  \caption{$\mathrm{Rel}\,H^{2s},\ N=16384$}
  \label{fig:feat_h2s_N16384}
\end{subfigure}\hfill
\begin{subfigure}[t]{0.44\linewidth}
  \centering
  \includegraphics[width=\linewidth,height=0.24\textheight,keepaspectratio]{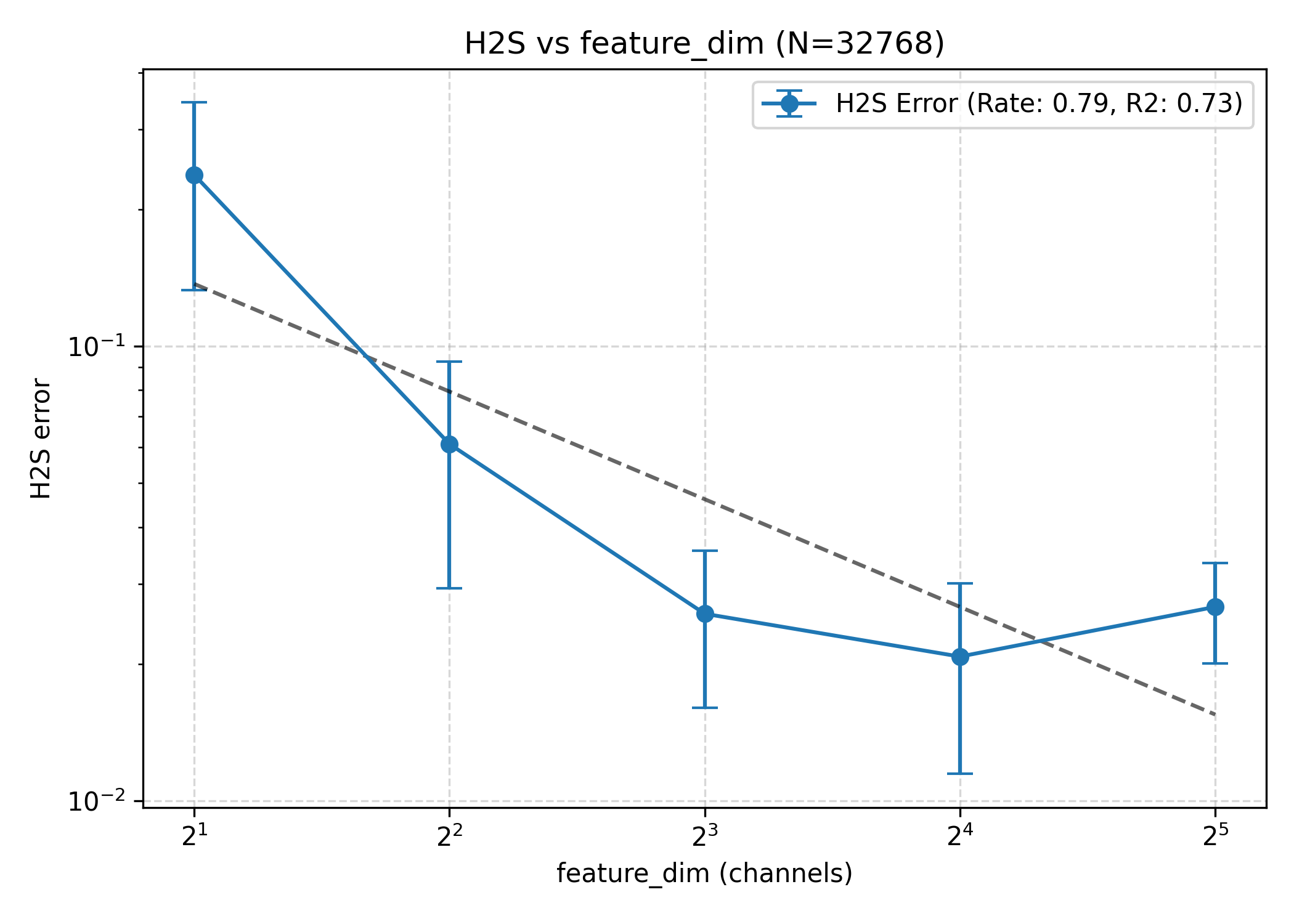}
  \caption{$\mathrm{Rel}\,H^{2s},\ N=32768$}
  \label{fig:feat_h2s_N32768}
\end{subfigure}

\caption{Channel-width sweep results on the upper-torus manifold $\mathcal M_{\mathbb T}$.}
\label{fig:feature_scaling_4panels}
\end{figure}
To isolate the effect of high-frequency boundary errors, we use manufactured solutions on the upper hemisphere:
\begin{equation*}
u_q(x,y,z)=(1+z)\operatorname{Im}\bigl((x+\mathrm{i}y)^q\bigr),
\qquad
u_q|_{z=0}=\sin(q\phi),
\end{equation*}
where \(\phi\) is the angular coordinate on the equatorial boundary. The trace contains a single Fourier mode \(q\), while the interior variation remains mild through the factor \(1+z\).

We use a CNN with \(L=3\) convolutional layers, channel width \(56\), MLP head \((24,8)\), and activation \(\mathrm{GeLU}^2\). The training set contains \(N=2048\) interior samples and \(M=256\) boundary samples. We sweep \(q\) and compare the standard \(L^2\) boundary loss with the spectral Sobolev boundary loss; the results are reported in~\autoref{tab:highfreq_q_sweep}.
\begin{table}[!htbp]
\centering
\small
\setlength{\tabcolsep}{4pt}
\renewcommand{\arraystretch}{1.15}
\caption{\(q\)-sweep for high-frequency boundary manufactured solutions. Values are reported as mean \(\pm\) std over five random seeds.}
\label{tab:highfreq_q_sweep}
\resizebox{\linewidth}{!}{%
\begin{tabular}{clcccc}
\hline
\(q\) & Boundary loss \(\mathcal L_{\mathrm{bnd}}\)
& Overall \(\mathrm{Rel}\,L^2\)
& Boundary \(\mathrm{Rel}\,L^2\)
& Overall reduction
& Boundary reduction \\
\hline
\multirow{2}{*}{2}
& \(L^2\)
& \(0.01661\pm0.02089\)
& \(0.02214\pm0.01795\)
& \multirow{2}{*}{\(2.25\times\)}
& \multirow{2}{*}{\(3.31\times\)} \\
& \(H^{2s-\frac12}\)
& \(0.007382\pm0.006089\)
& \(0.006696\pm0.005439\)
& & \\
\hline
\multirow{2}{*}{4}
& \(L^2\)
& \(0.03248\pm0.01560\)
& \(0.05407\pm0.03071\)
& \multirow{2}{*}{\(1.15\times\)}
& \multirow{2}{*}{\(2.44\times\)} \\
& \(H^{2s-\frac12}\)
& \(0.02820\pm0.03915\)
& \(0.02219\pm0.02863\)
& & \\
\hline
\multirow{2}{*}{6}
& \(L^2\)
& \(0.1973\pm0.2402\)
& \(0.3011\pm0.3100\)
& \multirow{2}{*}{\(5.55\times\)}
& \multirow{2}{*}{\(14.9\times\)} \\
& \(H^{2s-\frac12}\)
& \(0.03555\pm0.03845\)
& \(0.02016\pm0.01900\)
& & \\
\hline
\multirow{2}{*}{8}
& \(L^2\)
& \(0.5567\pm0.3894\)
& \(0.7247\pm0.4352\)
& \multirow{2}{*}{\(11.7\times\)}
& \multirow{2}{*}{\(22.1\times\)} \\
& \(H^{2s-\frac12}\)
& \(0.04774\pm0.03134\)
& \(0.03281\pm0.01119\)
& & \\
\hline
\end{tabular}}
\end{table}

~\autoref{tab:highfreq_q_sweep} shows that the advantage of the spectral Sobolev boundary loss becomes pronounced for high boundary frequencies. For \(q=8\), the overall \(\mathrm{Rel}\,L^2\) error decreases from \(0.5567\) to \(0.04774\), giving a \(11.7\times\) gain, while the boundary \(\mathrm{Rel}\,L^2\) error decreases from \(0.7247\) to \(0.03281\), giving a \(22.1\times\) gain. To explain this improvement, we examine the Fourier spectrum of the boundary error. The predicted and exact boundary traces are sampled on a dense uniform grid, and rFFT is applied to their difference. Since the exact boundary trace is dominated by the target mode \(q\), the spectrum reveals whether the error is concentrated near this mode or spread into spurious modes. We also report the off-target energy, defined as the relative Fourier energy outside the constant mode and the target mode. As shown in~\autoref{tab:boundary_fourier_diag}, for \(q=8\) this quantity decreases from \(0.7186\) to \(0.02368\), corresponding to a \(30.3\times\) reduction. This supports the interpretation that the Sobolev boundary loss suppresses non-target boundary oscillations through its frequency-weighted structure.

\begin{table}[!htbp]
\centering
\small
\setlength{\tabcolsep}{6pt}
\renewcommand{\arraystretch}{1.15}
\caption{Fourier diagnostic of the boundary error. The off-target energy measures the relative energy in non-target Fourier modes.}
\label{tab:boundary_fourier_diag}
\begin{tabular}{cccc}
\hline
\(q\) & \(L^2\) boundary loss & Sobolev boundary loss & Reduction \\
\hline
6 & \(0.2938\pm0.3116\) & \(0.01269\pm0.01040\) & \(23.1\times\) \\
8 & \(0.7186\pm0.4365\) & \(0.02368\pm0.01063\) & \(30.3\times\) \\
\hline
\end{tabular}
\end{table}

\begin{figure}[!htbp]
\centering
\includegraphics[width=\linewidth]{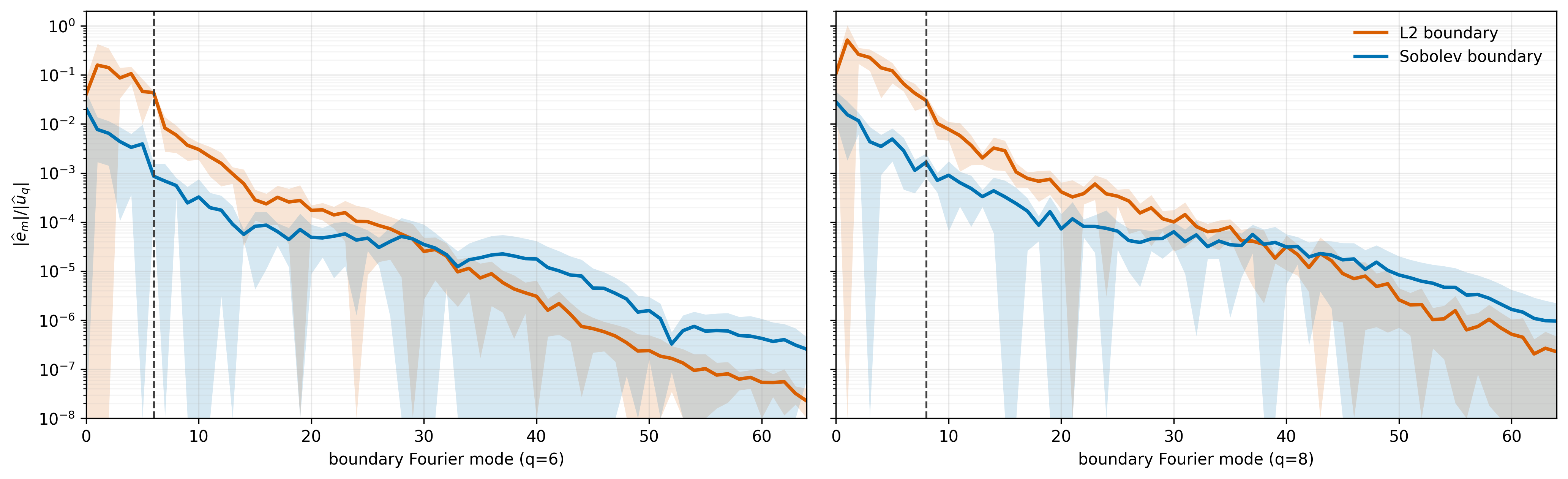}
\caption{Fourier spectra of the boundary error for \(q=6\) and \(q=8\). The modal errors are normalized by the amplitude of the exact target Fourier mode.}
\label{fig:boundary_fourier_modes}
\end{figure}
Finally, we examine the Fourier cutoff \(K\) in the spectral Sobolev boundary loss by fixing \(q=8\) and varying only the number of retained boundary modes; see~\autoref{tab:boundary_k_sweep}. The results show that a moderate cutoff already gives strong performance. Increasing \(K\) from \(8\) to \(16\) generally reduces the boundary error and the off-target energy, whereas \(K=32\) leads to a rebound in the error. This suggests a trade-off: too small a cutoff misses relevant boundary frequencies, while too large a cutoff may amplify the variance of the empirical boundary term and make the optimization less stable. In the present setting, \(K=12\) to \(K=16\) provides a relatively stable range.
\begin{table}[!htbp]
\centering
\small
\setlength{\tabcolsep}{8pt}
\renewcommand{\arraystretch}{1.15}
\caption{Fourier cutoff sweep with fixed \(q=8\). Values are reported as mean \(\pm\) std over five random seeds.}
\label{tab:boundary_k_sweep}
\begin{tabular}{ccc}
\hline
\(K\) & Boundary \(\mathrm{Rel}\,L^2\) & Off-target energy \\
\hline
8  & \(0.04188\pm0.006530\) & \(0.03975\pm0.009149\) \\
12 & \(0.03922\pm0.03800\)  & \(0.02794\pm0.02011\) \\
16 & \(0.02633\pm0.02021\)  & \(0.02094\pm0.01373\) \\
32 & \(0.03314\pm0.01467\)  & \(0.02659\pm0.01454\) \\
\hline
\end{tabular}
\end{table}
In summary, the experiments show that effective boundary regularization and suitable model capacity are both important for CNN-based solvers on manifolds with boundary. The Sobolev boundary penalty consistently reduces test errors and accelerates convergence compared with the standard PINN. For high-frequency Dirichlet data, the spectral Sobolev boundary loss further suppresses non-target Fourier modes, leading to lower boundary and overall errors. The \(K\)-sweep shows that a moderate number of boundary modes is already sufficient, while an overly large cutoff may increase the variance of the empirical boundary term and make optimization less stable. Finally, the channel-width sweep indicates that increasing the network width improves generalization, but the benefit saturates under a fixed data budget. These results suggest a practical trade-off between boundary spectral control, network capacity, and optimization cost.

\section{Conclusion}\label{Section: Conclusion}

We developed a simultaneous approximation framework for convolutional neural networks on compact Riemannian manifolds. The analysis is based on restricted kernels and their induced Sobolev/RKHS structures. Since the restricted kernels may be singular at the origin but analytic away from it, we decompose them into singular and analytic components and approximate each part separately. This yields CNN approximation rates governed by the intrinsic dimension of the manifold. We also show that a multichannel construction avoids the depth-related instability of one-channel CNNs and gives sharper complexity control.

We further applied this approximation framework to PICNNs for elliptic boundary value problems on manifolds with boundary. Guided by a priori estimates for elliptic boundary value problems, we replaced the standard low-order boundary penalty with a spectral Sobolev boundary loss, implemented through boundary eigenfunction expansions or FFT-based estimators. We also provided an error-decomposition analysis that separates the main sources of error, including approximation error, generalization errors arising from interior and boundary sampling, spectral truncation error, and optimization effects. 

Numerical results on both the hemisphere and half-torus validate the proposed Sobolev boundary treatment for low- and high-frequency benchmark problems. For low-frequency solutions, our method consistently achieves smaller overall errors and exhibits faster convergence during training, reflecting the advantage of incorporating higher-order boundary regularity into the learning objective. We also conducted architecture studies by varying the network capacity, including the number of channels and channel width, to examine how model complexity influences approximation quality. The observed trends are consistent with the approximation theory and further support the effectiveness of the proposed CNN architecture on manifolds. For high-frequency solutions, it significantly reduces both boundary and overall errors, often by one to two orders of magnitude. The Fourier cutoff study further confirms the predicted trade-off associated with spectral truncation, where too small a cutoff misses relevant high-frequency boundary information, while too large a cutoff amplifies sampling noise. These results show that boundary spectral control, model capacity, and optimization accuracy jointly determine the performance of PICNN solvers on manifolds.

Several problems remain open. A natural next step is to turn the present error decomposition into a full convergence-rate theorem by combining elliptic stability, boundary spectral truncation, and statistical estimates for the interior and boundary empirical losses. Another direction is to extend the RKHS-based approximation framework beyond the native-space regime, especially to weak-solution settings and lower-regularity PDEs on manifolds.

\section*{Appendix}
\addcontentsline{toc}{section}{Appendix}
\label{allapp}
\appendix
\section{Notations}\label{Appendix: Notations}

For positive sequences $A_n$ and $B_n$, we write $A_n \lesssim B_n$ if $A_n$ is bounded by $B_n$ up to a constant (independent of $n$). We write $A_n \asymp B_n$ if both $A_n \lesssim B_n$ and $B_n \lesssim A_n$ hold.

We consider a compact, connected, and complete $d$-dimensional Riemannian manifold $(\mathcal{M}^{d}, g)$ of bounded geometry. The metric $g$ defines an inner product on the tangent bundle $T\mathcal{M}^{d}$. In a local coordinate chart $(x^i)$, with natural basis $\partial_i$, the metric components are $g_{ij} = \langle \partial_i, \partial_j \rangle$, and the volume element is $dV_g = \sqrt{\det(g)}\, dx$.

The Levi-Civita connection $\nabla$ is the unique torsion-free connection satisfying $\nabla g = 0$. Its action on vector fields is given by $\nabla_j X^i = \partial_j X^i + \Gamma^i_{kj} X^k$, where $\Gamma^i_{kj} = \frac{1}{2} g^{il} ( \partial_k g_{lj} + \partial_j g_{lk} - \partial_l g_{kj} )$. The gradient of a function $f$ is $\mathrm{grad}_g f = g^{ij} \partial_j f \partial_i$, and the divergence of a vector field $X$ is $\mathbf{div}_g X = \nabla_i X^i = (\det g)^{-1/2} \partial_i ( (\det g)^{1/2} X^i )$. Consequently, the Laplace-Beltrami operator is defined as $\Delta_{\mathcal{M}} f := -\mathbf{div}_g(\mathrm{grad}_g f)$.

The operator $\Delta_{\mathcal{M}}$ admits a unique self-adjoint extension on $L^2(\mathcal{M}^d)$. By the spectral theorem, $\Delta_{\mathcal{M}} = \int_{0}^\infty \lambda \, dP(\lambda)$. For any Borel function $\Phi: [0,\infty) \to \mathbb{R}$, the operator $\Phi(\Delta_{\mathcal{M}})$ is defined via the spectral integral:
\begin{equation*}
    \langle \Phi(\Delta_{\mathcal{M}})f, h \rangle_{L^2} = \int_{0}^{\infty} \Phi(\lambda) \, dP_{f,h}(\lambda),
\end{equation*}
for all functions $f$ in the domain $\mathrm{dom}(\Phi(\Delta_{\mathcal{M}})) = \{ f \in L^2(\mathcal{M}^d) \mid \int_0^{\infty} |\Phi(\lambda)|^2 \, dP_{f,f}(\lambda) < \infty \}$.

Finally, we state the regularity assumption on the manifold.

\begin{definition}[Bounded geometry]\label{def:bounded_geometry}
A complete Riemannian manifold $\mathcal{M}^d$ has bounded geometry if:
\begin{enumerate}[(i)]
\item The injectivity radius is bounded from below: $\mathrm{inj}(\mathcal{M}^d) \ge \delta > 0$.
\item In any normal coordinate chart, the metric tensor $g_{ij}$ and its derivatives are uniformly bounded; that is, for all multi-indices $\alpha$, there exists $C_\alpha$ such that $|\partial^\alpha g_{ij}| \le C_\alpha$ independent of the chart.
\end{enumerate}
\end{definition}

\section{Supplement for approximation error analysis}\label{Appendix: Supplement for Approximation Error Analysis}
We first present the proof of \autoref{theorem: simultaneous approximation}. 
\subsection{Proof of~\autoref{theorem: simultaneous approximation}}\label{Appendix: Proof of simultaneous approximation 1channel}
\begin{proof}[Proof of~\autoref{theorem: simultaneous approximation}]\label{Proof: Proof of simultaneous approximation with 1 channel}
We recall that there exist quasi-uniform point sets $X=\{x_i\}_{i=1}^N\subset\mathcal M^d$ with fill distance $h_{X,\mathcal M^d}\lesssim N^{-1/d}$; see~\cite[Proposition~3]{azangulov2024convergence}. Hence, by~\cite[Theorem~10]{fuselier2012scattered}, for sufficiently large $N$, the kernel interpolant $I_Xf$ satisfies
\begin{equation}\label{eq:rbf_rate}
\|f-I_X f\|_{\mathcal H^s(\mathcal M^d)}
\le C\,N^{-(k-s)/d}\,\|f\|_{\mathcal H^k(\mathcal M^d)},
\qquad 0\le s<k.
\end{equation}
Writing $\Psi$ for the restricted kernel, the interpolant $I_Xf$ takes the form
\begin{equation*}
    I_X f(x)=\sum_{i=1}^N c_i\,\Psi(x,x_i), \qquad \text{with } I_X f(x_j)=f(x_j),\ j=1,\dots,N.
\end{equation*}
By \autoref{lemma: approximation of restricted kernel function}, for any $\varepsilon'>0$ and each $i=1,\dots,N$,
there exists a CNN $F_i^{(L+L_0+1)}\in\mathcal F$ such that
\begin{equation*}
    \|\Psi(\cdot,x_i)-F_i^{(L+L_0+1)}\|_{\mathcal H^s(\mathcal M^d)}\le \varepsilon'.
\end{equation*}
Moreover, the network complexity satisfies
\begin{equation*}
    L \lesssim \Big\lceil \frac{n_2D-1}{S-2}\Big\rceil,\quad
    L_0 \lesssim \log \log (1/\varepsilon'),\quad
    d_l \lesssim \log^2(1/\varepsilon'),\quad
    \mathcal S \lesssim \log^3(1/\varepsilon').
\end{equation*}
Define the linear-combination network
\begin{equation*}
    F^{(L+L_0+1)}(x):=\sum_{i=1}^N c_i\,F_i^{(L+L_0+1)}(x).
\end{equation*}
Following \autoref{lemma: downsampling}, we first generate the features $\xi_{i,j}$, $j=1,\ldots,n_2$, associated with each $\{x_i\}_{i=1}^N$, yielding $F^{L}(x)=\bigl(\langle x,\xi_{i,j}\rangle\bigr)_{i=1,\dots,N,\;j=1,\ldots,n_2}$. The first $\mathrm{ReQU}$ layer then produces the quadratic features $F^{L+1}(x)=\bigl(\|x-y_i\|_2^2\bigr)_{i=1,\ldots,N}$ by \autoref{lemma: feature construction} and \autoref{lemma: multiply}. Finally, \autoref{lemma: approximation of restricted kernel function} gives $F^{L+L_0+1}(x)=\sum_{i=1}^N c_i\,\widetilde{\phi}(\|x-y_i\|_2^2)=\sum_{i=1}^N c_i\,F_i^{(L+L_0+1)}(x)$. Because the same approximation $\widetilde{\phi}$ is used for every $i$, this step introduces only logarithmically many free parameters, and hence the total parameter complexity remains dominated by the CNN layers. The complexity of the resulting network is characterized by:
\begin{align*}
    L &\lesssim \Big\lceil \frac{n_2ND-1}{S-2}\Big\rceil,\qquad L_0 \lesssim \log \log(1/\varepsilon),\\
    d_\ell &\lesssim N\log^2(1/\varepsilon'),\qquad \ell=L+1,\dots,L+L_0+1,
\end{align*}
and $\mathcal S \lesssim N.$ Then, by the triangle inequality,
\begin{equation*}
    \|f-F^{(L+L_0+1)}\|_{\mathcal H^s(\mathcal M^d)} \le \|f-I_X f\|_{\mathcal H^s(\mathcal M^d)} +\|I_X f-F^{(L+L_0+1)}\|_{\mathcal H^s(\mathcal M^d)}.
\end{equation*}
The first term is bounded by~\eqref{eq:rbf_rate}. For the second term,
\begin{equation*}
    \|I_X f-F^{(L+L_0+1)}\|_{\mathcal H^s(\mathcal M^d)} \le \sum_{i=1}^N |c_i|\,\|\Psi(\cdot,x_i)-F_i^{(L+L_0+1)}\|_{\mathcal H^s(\mathcal M^d)} \le \|c\|_1\,\varepsilon'.
\end{equation*}
Since $c=K(X)^{-1}f(X)$, we have $\|c\|_1\le \sqrt{N}\|c\|_2\le \sqrt{N}\,\|K(X)^{-1}\|_{\mathrm{op}}\,\|f(X)\|_2$. To bound $\|K(X)^{-1}\|_{\mathrm{op}}$, we use the connection between the native space of the ambient kernel and that of its restriction to $M$ established in~\cite{fuselier2012scattered}. In particular, the native space of the restricted kernel is norm-equivalent to $\mathcal H^k(\mathcal M)$. It follows that, for any $a\in\mathbb R^N$, the quadratic form $a^\top K(X)a$ is equivalent to the $\mathcal H^{-k}(\mathcal M)$-norm squared of the discrete measure $\sum_{j=1}^N a_j\delta_{x_j}$. By testing this dual norm against suitably chosen bump functions supported in pairwise disjoint geodesic balls centered at the data sites, one obtains $a^\top K(X)a\gtrsim q_X^{\,2k-d}\|a\|_2^2$, and hence $\lambda_{\min}(K(X))\gtrsim q_X^{\,2k-d}$, where $q_X$ is the separation radius of $X$. Therefore, $\|K(X)^{-1}\|_{\mathrm{op}}=\lambda_{\min}(K(X))^{-1}\lesssim q_X^{-(2k-d)}$. For quasi-uniform point sets, $q_X\asymp N^{-1/d}$, so $\|K(X)^{-1}\|_{\mathrm{op}}\lesssim N^{(2k-d)/d}$. Consequently, $\|c\|_1\lesssim N^{1/2+(2k-d)/d}\,\|f(X)\|_2$. On the other hand, by the reproducing property and the boundedness of the diagonal kernel, $|f(x_i)|\le \|f\|_{\mathcal H^k(\mathcal M^d)}\|k(x_i,\cdot)\|_{\mathcal H^k(\mathcal M^d)}\lesssim \|f\|_{\mathcal H^k(\mathcal M^d)}$ for each $i$, and hence $\|f(X)\|_2\lesssim \sqrt{N}\,\|f\|_{\mathcal H^k(\mathcal M^d)}$. Combining these estimates yields $\|c\|_1\lesssim N^{1/2}\cdot N^{(2k-d)/d}\cdot N^{1/2}\,\|f\|_{\mathcal H^k(\mathcal M^d)}=N^{2k/d}\|f\|_{\mathcal H^k(\mathcal M^d)}$. Consequently,
\begin{equation*}
    \|f-F^{(L+L_0+1)}\|_{\mathcal H^s(\mathcal M^d)} \lesssim N^{-(k-s)/d} + N^{2k/d}\,\varepsilon'.
\end{equation*}
Choosing $\varepsilon'\asymp N^{-2k/d-(k-s)/d}$ yields
\begin{equation*}
    \|f-F^{(L+L_0+1)}\|_{\mathcal H^s(\mathcal M^d)}\lesssim N^{-(k-s)/d}.
\end{equation*}
Equivalently, for any $\varepsilon>0$, taking $N\asymp \varepsilon^{-d/(k-s)}$ gives
\begin{equation*}
    \|f-F^{(L+L_0+1)}\|_{\mathcal H^s(\mathcal M^d)}\lesssim \varepsilon.
\end{equation*}
With this choice, the network complexity follows from the previous bounds, completing the proof.
\end{proof}

\begin{proof}[Proof of \autoref{thm:mult-inner-products}]
The proof of this theorem differs from that of~\autoref{theorem: simultaneous approximation} primarily in the feature construction. We depart from the previous wavelet-based approach and leverage the parallel nature of multiple channels to construct features directly. This strategy avoids depth explosion, thereby yielding a tighter upper bound on the network complexity. The construction logic for an individual feature follows the same approach as in~\cite[Lemma A.4]{yang2025rates}. 

Let $\{\xi_{i,j}\}_{1\le i\le N,\ 1\le j\le n_2}\subset \mathbb R^D,$ be the feature constructed in \autoref{theorem: simultaneous approximation}, where \(N\) is the number of samples on manifold and \(n_2\) is the constant from~\autoref{lemma: feature construction}. Let $m:=Nn_2$, Rearrange this family as $\{\xi^{(r)}\}_{r=1}^m,\xi^{(r)}:=\xi_{i,j}\,\text{whenever }r=(i-1)n_2+j.$ 
For each \(r\in[m]\), write $\xi^{(r)}=(\xi^{(r)}_1,\dots,\xi^{(r)}_D)^\top,$ and allocate three channels $I_r:=\{3r-2,\,3r-1,\,3r\}$ to the \(r\)-th branch. The first two channels store the positive and negative parts of the running partial sum associated with \(\langle \xi^{(r)},x\rangle\), while the third channel shifts the unused tail of the input. Define \(w^{(0)}\in\mathbb R^{s\times 3m\times 1}\) and \(b^{(0)}=0\) by
\[
w^{(0)}_{:,\,3r-2,\,1}
=
\begin{pmatrix}
\xi^{(r)}_1\\
\vdots\\
\xi^{(r)}_s
\end{pmatrix},
\qquad
w^{(0)}_{:,\,3r-1,\,1}
=
-\begin{pmatrix}
\xi^{(r)}_1\\
\vdots\\
\xi^{(r)}_s
\end{pmatrix},
\qquad
w^{(0)}_{:,\,3r,\,1}
=
e_s:=
\begin{pmatrix}
0\\
\vdots\\
0\\
1
\end{pmatrix},
\]
for \(r=1,\dots,m\), and all remaining entries of \(w^{(0)}\) are set to zero. Since the one-sided convolution by \(e_s\) is the left translation by \(s-1\), the first hidden layer satisfies
\[
F^{(1)}_{1,\,3r-2}(x)
=
\sigma\!\left(\sum_{k=1}^{s} \xi^{(r)}_k x_k\right),
\qquad
F^{(1)}_{1,\,3r-1}(x)
=
\sigma\!\left(-\sum_{k=1}^{s} \xi^{(r)}_k x_k\right),
\]
and
\[
F^{(1)}_{:,\,3r}(x)
=
(x_s,x_{s+1},\dots,x_D,0,\dots,0)^\top.
\]

For \(\ell=1,\dots,L-1\), define \(w^{(\ell)}\in\mathbb R^{s\times 3m\times 3m}\) and \(b^{(\ell)}=0\). The tensor \(w^{(\ell)}\) is block diagonal with respect to the channel groups \(I_r\), namely
\[
w^{(\ell)}_{:,j',j}=0
\qquad\text{whenever }j\in I_r,\ j'\in I_{r'},\ r\neq r'.
\]
For each \(r\in[m]\), let $q_\ell:=\min\{\ell(s-1)+1,D\}.$ On the block \(I_r\times I_r\), define
\[
w^{(\ell)}_{:,\,3r-2,\,I_r}
=
\begin{pmatrix}
1 & -1 & 0\\
0 & 0 & \xi^{(r)}_{\ell(s-1)+2}\\
\vdots & \vdots & \vdots\\
0 & 0 & \xi^{(r)}_{q_{\ell+1}}\\
0 & 0 & 0\\
\vdots & \vdots & \vdots\\
0 & 0 & 0
\end{pmatrix},
\qquad
w^{(\ell)}_{:,\,3r-1,\,I_r}
=
-
\begin{pmatrix}
1 & -1 & 0\\
0 & 0 & \xi^{(r)}_{\ell(s-1)+2}\\
\vdots & \vdots & \vdots\\
0 & 0 & \xi^{(r)}_{q_{\ell+1}}\\
0 & 0 & 0\\
\vdots & \vdots & \vdots\\
0 & 0 & 0
\end{pmatrix},
\]
and
\[
w^{(\ell)}_{:,\,3r,\,I_r}
=
\begin{pmatrix}
0&0&0\\
\vdots&\vdots&\vdots\\
0&0&0\\
0&0&1
\end{pmatrix}.
\]
We claim that after the \(\ell\)-th hidden layer,
\[
F^{(\ell)}_{1,\,3r-2}(x)
=
\sigma\!\left(\sum_{k=1}^{q_\ell} \xi^{(r)}_k x_k\right),
\qquad
F^{(\ell)}_{1,\,3r-1}(x)
=
\sigma\!\left(-\sum_{k=1}^{q_\ell} \xi^{(r)}_k x_k\right),
\]
and
\[
F^{(\ell)}_{:,\,3r}(x)
=
(x_{q_\ell},x_{q_\ell+1},\dots,x_D,0,\dots,0)^\top.
\]
The case \(\ell=1\) follows from the construction of the first layer. Assume now that the claim holds for some \(\ell\in[L-1]\). Using the identity
\[
t=\sigma(t)-\sigma(-t),
\]
the first coordinate of channel \(3r-2\) at layer \(\ell+1\) becomes
\[
F^{(\ell)}_{1,\,3r-2}(x)-F^{(\ell)}_{1,\,3r-1}(x)
+\sum_{k=q_\ell+1}^{q_{\ell+1}} \xi^{(r)}_k x_k
=
\sum_{k=1}^{q_{\ell+1}} \xi^{(r)}_k x_k,
\]
while the first coordinate of channel \(3r-1\) becomes its negative counterpart, and the third channel is shifted again by \(s-1\). Hence, the claim follows by induction. Since
\[
L=\Big\lceil\frac{D-1}{s-1}\Big\rceil,
\]
we have \(q_L=D\), and therefore
\[
F^{(L)}_{1,\,3r-2}(x)=\sigma(\langle \xi^{(r)},x\rangle),
\qquad
F^{(L)}_{1,\,3r-1}(x)=\sigma(-\langle \xi^{(r)},x\rangle).
\]
For each \(r\in[m]\), choose the output matrix \(W^{(L,r)}\in\mathbb R^{D\times 3m}\) by
\[
W^{(L,r)}_{1,\,3r-2}=1,\qquad
W^{(L,r)}_{1,\,3r-1}=-1,
\]
and all remaining entries equal to zero. Then
\[
\langle W^{(L,r)},F^{(L)}(x)\rangle_F
=
F^{(L)}_{1,\,3r-2}(x)-F^{(L)}_{1,\,3r-1}(x)
=
\langle \xi^{(r)},x\rangle.
\]
Thus, we can compute the desired inner products in parallel $\bigl(\langle \xi^{(1)},x\rangle,\dots,\langle \xi^{(m)},x\rangle\bigr)$, which is exactly the desired vector after identifying \(\xi^{(r)}=\xi_{i,j}\) through \(r=(i-1)n_2+j\). The subsequent operations in the fully connected layers are identical to those in~\autoref{theorem: simultaneous approximation}, namely the construction of ridge polynomials and the high-order approximation of Sobolev kernels, which we omit here for brevity. Moreover, because the network is given by \(m\) parallel channel blocks with block-diagonal convolutional kernels and no cross-branch coupling, the total number of nonzero free parameters is the sum of the contributions of the individual branches. Hence, for fixed \(D\) and \(S\), it is of order \(O(m)\). The proof is thus complete.
\end{proof}

\subsection{Approximation of restricted kernel function}\label{Appendix: Approximation of restricted kernel function}
In this subsection, we show that our CNN efficiently approximates restricted kernels on a manifold in Sobolev norms. This is a foundational step for our main theorem. The core challenge is the slight singularity of Matérn kernels at $r=0$. To handle this, we use smooth spline-based cutoffs to avoid large high-order errors. $\mathrm{ReLU}$ alone isn't smooth enough, but skipping the cutoff entirely would make the model too complex due to the singularity. We write $\sigma_2(t):=(t)_+^2$ for the $\mathrm{ReQU}$ activation.

\begin{lemma}\label{lemma: approximation of restricted kernel function}
Let $\mathcal{M}^d\subset \mathbb{R}^D$ satisfy the assumptions in~\autoref{Section: Preliminaries}.
For any $\varepsilon>0$ and $0\le s<k$, there exists a CNN $F^{(L+L_0+1)}\in\mathcal{F}$ with depth
\[
L \lesssim \left\lceil \frac{n_2D-1}{S-2}\right\rceil,\qquad
L_0 \lesssim \log \log (1/\varepsilon),
\]
where $n_2$ comes from~\autoref{lemma: feature construction}, width $d_\ell\lesssim \log^2(1/\varepsilon)$, and number of parameters
$\mathcal{S}\lesssim \log^3(1/\varepsilon)$, such that for any fixed $y\in\mathcal{M}^d$,
\[
\|\Psi(\cdot,y)-F^{(L+L_0+1)}(\cdot,y)\|_{\mathcal{H}^s(\mathcal{M}^d)}\le \varepsilon.
\]
\end{lemma}
We collect the lemmas that will be used to explicitly build the required network.
\begin{lemma}[Lemma 5 in~\cite{zhou2018deep}]
\label{lemma: feature construction}
Let $D\in\mathbb{N}$ and $l\in\mathbb{N}$.
There exists an integer $n_l$ and unit vectors $\{\xi_i\}_{i=1}^{n_l}\subset \mathbb{S}^{D-1}$ depending only on $(D,l)$ such that for any polynomial
$P_l\in \mathcal{P}_l(\mathbb{R}^D)$ of degree at most $l$, there exist univariate polynomials
$\{p_{i,l}\}_{i=1}^{n_l}\subset \mathcal{P}_l(\mathbb{R})$ satisfying
\[
P_l(x)=\sum_{i=1}^{n_l} p_{i,l}(\xi_i\cdot x),\qquad x\in\mathbb{R}^D.
\]
\end{lemma}

\begin{lemma}[Lemma 3 in~\cite{fang2020theory}]
\label{lemma: downsampling}
Let $m\in\mathbb{N}$ and let $y=\{y_1,\ldots,y_m\}\subset \mathbb{S}^{D-1}$ be a set of features.
There exist convolution kernels $\{w^{(\ell)}\}_{\ell=1}^{L}$ with identical filter size $S$ and depth
\[
L \le \left\lceil \frac{mD-1}{S-2}\right\rceil,
\]
and biases $b^{(\ell)}$ such that the CNN can output the feature inner-products (up to a constant shift):
\[
\mathcal{D}\big(F^{(L)}(x)\big)=
\begin{bmatrix}
\langle x,y_1\rangle\\
\vdots\\
\langle x,y_m\rangle\\
0\\
\vdots\\
0
\end{bmatrix}
+ B^{(L)}\mathbf{1},
\qquad
B^{(L)}=\prod_{\ell=1}^L \|w^{(\ell)}\|_1 .
\]
\end{lemma}

\begin{lemma}[Lemma 2 in~\cite{zhou2025expressive}]
\label{lemma: multiply}
    Let $n \in \mathbb{N}_+, n\geq 2$. For any $x =  (x_1,\ldots,x_n) \in \mathbb{R}^n$, there exists a $\mathrm{ReLU}^k$ network $\mathcal{F}^n_{mult}$ with depth $\lceil \log n \rceil$, width $k2^{\lceil \log n \rceil}$, and number of parameters $k2^{2\lceil \log n \rceil+2}\lceil M(2,k-1) \rceil$ that exactly computes the product $\mathcal{F}^n_{mult}(x) = \prod_{i=1}^nx_i$.
\end{lemma}
\begin{proof}[Proof of~\autoref{lemma: approximation of restricted kernel function}]
We first decompose the Mat\'ern kernel into analytic parts plus an explicit singular factor, then smooth-truncate and approximate the singular factor in the $u=r^2$ variable with a controlled Sobolev error, and finally realize the resulting building blocks and products by a CNN with the stated complexity.

Consider a Mat\'ern-type kernel $\phi(x,y)=\phi(r)$ with $r=\|x-y\|_2$.
By~\cite[Theorem 6.13]{wendland2004scattered}, for $\tau>D/2$,
\[
\phi(r)=\frac{2^{1-\tau}}{\Gamma(\tau)}\,r^{\tau-D/2}\,K_{D/2-\tau}(r),
\]
where $K_\alpha$ is the modified Bessel function of the second kind.
Let $\nu:=\tau-D/2>0$.
Using the identity $K_\nu(r)=\frac{\pi}{2\sin(\pi\nu)}\big(I_{-\nu}(r)-I_\nu(r)\big)$ for $\nu\notin\mathbb{Z}$ and the power series of $I_{\pm\nu}$,
one obtains the standard decomposition
\[
r^{\nu}K_{\nu}(r)=A(r^2)+r^{2\nu}B(r^2)\qquad(\nu\notin\mathbb{Z}),
\]
where $A(\cdot)$ and $B(\cdot)$ are analytic on $[0,M]$ with $M:=\mathrm{diam}(\mathcal{M}^d)^2$.
For integer $\nu\in\mathbb{N}$ the decomposition becomes
\[
r^{\nu}K_\nu(r)=\widetilde{A}(r^2)+r^{2\nu}\log(r)\,\widetilde{B}(r^2),
\]
with analytic $\widetilde{A},\widetilde{B}$ on $[0,M]$.
Hence, after absorbing constants into analytic parts, we may write
\[
\phi(r)=
\begin{cases}
A(r^2)+r^{2\nu}B(r^2), & \nu\notin\mathbb{N},\\[2mm]
\widetilde{A}(r^2)+r^{2\nu}\log(r)\,\widetilde{B}(r^2), & \nu\in\mathbb{N}.
\end{cases}
\]
The only non-smoothness stems from the factors $r^{2\nu}$ (or $r^{2\nu}\log r$) at $r=0$.

Set $u=r^2\in[0,M]$ and rewrite the singular factor as $u^\nu$ (or $u^\nu\log u$).
We construct a $C^{p}$ cutoff $\chi_p$ by integrating a cardinal B-spline.
Let
\[
M_p(t)=\frac{1}{p!}\sum_{j=0}^{p+1}(-1)^j\binom{p+1}{j}(t-j)_+^{p},
\qquad \mathrm{supp}(M_p)\subset[0,p+1],\quad M_p\in C^{p-1}.
\]
Define its integral (a smooth step)
\[
\beta_p(t):=\int_{-\infty}^{t} M_p(s)\,ds
=\frac{1}{(p+1)!}\sum_{j=0}^{p+1}(-1)^j\binom{p+1}{j}(t-j)_+^{p+1},
\]
so that $\beta_p(t)=0$ for $t\le 0$ and $\beta_p(t)=1$ for $t\ge p+1$, with $\beta_p\in C^{p}$.
Rescale it to the transition window $[1,2]$ by
\[
B_p(t):=\beta_p\big((p+1)(t-1)\big),\qquad
\chi_p(t):=1-B_p(t).
\]
Then $\chi_p\equiv 1$ on $(-\infty,1]$, $\chi_p\equiv 0$ on $[2,\infty)$, and $\chi_p\in C^{p}$.
Moreover, $\chi_p$ is an explicit linear combination of shifted $\sigma_{p+1}$ (hence also representable by $\sigma_k$ for any $k\ge p+1$).

Fix a truncation threshold $\eta\in(0,1)$ in the $u$-variable.
For the analytic functions $A(\cdot),B(\cdot)$, polynomial approximation yields exponential convergence:
there exists $\rho\in(0,1)$ such that for each $s'\in\mathbb{N}$,
\[
\|A-P_A^N\|_{W_\infty^{s'}([0,M])}\lesssim \rho^{N+1-s'},
\qquad
\|B-P_B^N\|_{W_\infty^{s'}([0,M])}\lesssim \rho^{N+1-s'} .
\]
For $u^\nu$ on $[\eta,M]$, using the dyadic partition and switch-based parallelization technique in~\cite[Lemma B.1]{oko2023diffusion}, we can construct a network $P_{u^\nu}$ such that
\[
\|u^\nu - P_{u^\nu}\|_{W_\infty^{s'}([\eta,M])} \lesssim \varepsilon .
\]
Moreover, letting $J=\lceil \log (M/\eta)\rceil$, the network can be chosen with depth $L=\mathcal O(\log\log \varepsilon^{-1})$, width $W=\mathcal O(J\,\log \varepsilon^{-1})$, and total number of parameters $S=\mathcal O(J\,\log \varepsilon^{-1})$.
Near $u=0$ we replace $u^\nu$ by a smooth polynomial patch.
We take $Q_{\eta,m}$ as the degree-$m$ Taylor polynomial of $u^\nu$ around $u=\eta$
(which is valid on $[0,2\eta]$), and define the truncated approximation of $u^\nu$ by
\[
\widetilde{u^\nu}(u)
:=\chi_p\!\left(\frac{u}{\eta}\right) Q_{\eta,m}(u)
+\Bigl(1-\chi_p\!\left(\frac{u}{\eta}\right)\Bigr) P_{u^\nu}^N(u),
\qquad p\ge s'+1 .
\]
Finally define
\[
\widetilde{\phi}(u):=P_A^N(u)+\widetilde{u^\nu}(u)\,P_B^N(u),
\qquad
\widetilde{\phi}(r):=\widetilde{\phi}(r^2).
\]

A direct Sobolev estimate shows that the truncation error is supported in the thin neighborhood $\{u\lesssim \eta\}$,
equivalently $\{r\lesssim \sqrt{\eta}\}$.
Consequently, for $s'<2\nu+\frac{D}{2}$,
\begin{equation}\label{eq:trunc_error_u_eta}
\|\widetilde{\phi}-\phi\|_{H^{s'}(\Omega)}
\;\lesssim\;
\eta^{\frac{2\nu+\frac{D}{2}-s'}{2}}
+\eta^{-2s'}\rho^{N+1-s'},
\end{equation}
where $\Omega$ is any bounded domain containing $\mathcal{M}^d$. Choose
\[
\eta \asymp \varepsilon^{\frac{2}{2\nu+\frac{D}{2}-s'}},
\qquad
N \asymp \log(1/\varepsilon)+s',
\]
then \eqref{eq:trunc_error_u_eta} gives $\|\widetilde{\phi}-\phi\|_{H^{s'}(\Omega)}\lesssim \varepsilon$.
By the trace theorem to a $d$-dimensional smooth manifold $\mathcal{M}^d$,
\[
\|\widetilde{\phi}-\phi\|_{\mathcal{H}^{s'-(D-d)/2}(\mathcal{M}^d)}\lesssim \varepsilon,
\qquad s'>\frac{D-d}{2}.
\]
Since $\tau>D/2$ and $\nu=\tau-D/2$, the condition $s'<2\nu+D/2$ becomes $s'<2\tau-D/2$,
which allows us to reach the required Sobolev orders up to $k$ in the manifold norm.

We now construct a CNN that computes the building blocks appearing in $\widetilde{\phi}(u)$ with the desired complexity.

\emph{(a) Computing $u=\|x-y\|_2^2$.}
The function $x\mapsto \|x-y\|_2^2$ is a degree-$2$ polynomial in $x$.
By~\autoref{lemma: feature construction} with $l=2$, it can be written as a sum of univariate polynomials of ridge features
$p_{i,2}(\xi_i\cdot x)$.
By~\autoref{lemma: downsampling}, the CNN front-end extracts $\{\xi_i\cdot x\}_{i=1}^{n_2}$ with depth
$L\lesssim \lceil (n_2D-1)/(S-2)\rceil$.
Then each univariate polynomial $p_{i,2}$ can be represented exactly by a $\mathrm{ReQU}$ network
using~\autoref{lemma: multiply} (and linear combinations).
Summing these branches in parallel yields a network block computing $F_{\mathrm{sq}}(x;y)=\|x-y\|_2^2.$

\emph{(b) Approximating analytic polynomials $P_A^N(u),P_B^N(u)$ and the cutoff $\chi_p(u/\eta)$.}
Each of these is a univariate polynomial (or a fixed spline combination of shifted $\sigma_{p+1}$).
Using repeated composition of $\mathrm{ReQU}$ layers (standard power-building), one can represent monomials up to degree
$N\asymp \log(1/\varepsilon)$ with depth $L_0\lesssim \log\log(1/\varepsilon)$ and width $O(\log^2(1/\varepsilon))$,
and then represent the required polynomials by linear combinations.
The cutoff $\chi_p(\cdot)$ is itself an explicit linear combination of shifted $\sigma_{p+1}$, hence realizable by constant layers.

\emph{(c) Products.}
Finally, we implement the products in
$\widetilde{\phi}(u)=P_A^N(u)+\widetilde{u^\nu}(u)\,P_B^N(u)$
using the exact multiplication network in~\autoref{lemma: multiply}.

Combining (a)--(c) yields a network $F^{(L+L_0+1)}(\cdot,y)$ such that
\[
\|\Psi(\cdot,y)-F^{(L+L_0+1)}(\cdot,y)\|_{\mathcal{H}^{s}(\mathcal{M}^d)}\le \varepsilon,
\]
with depth/width/parameter counts as stated in~\autoref{lemma: approximation of restricted kernel function}.
This completes the proof.
\end{proof}

\section{Supplement for error decomposition analysis}\label{appendix: error decomposition}
We present the proof of~\autoref{theorem:picnn_error_decomposition}.

\begin{proof}\label{proof:picnn_error_decomposition}
By the minimizing property of \(\widehat u_{\mathcal F,K}\),
\[
    \widehat{\mathcal J}_K(\widehat u_{\mathcal F,K})
    \le
    \widehat{\mathcal J}_K(u_{\mathcal F,K}).
\]
Adding and subtracting the population functional gives
\[
\begin{aligned}
\mathcal J_K(\widehat u_{\mathcal F,K})
&=
\widehat{\mathcal J}_K(\widehat u_{\mathcal F,K})
+
\left[
\mathcal J_K(\widehat u_{\mathcal F,K})
-
\widehat{\mathcal J}_K(\widehat u_{\mathcal F,K})
\right]
\\
&\le
\widehat{\mathcal J}_K(u_{\mathcal F,K})
+
G_K(\widehat u_{\mathcal F,K})
\\
&=
\mathcal J_K(u_{\mathcal F,K})
+
\left[
\widehat{\mathcal J}_K(u_{\mathcal F,K})
-
\mathcal J_K(u_{\mathcal F,K})
\right]
+
G_K(\widehat u_{\mathcal F,K}).
\end{aligned}
\]
It remains to control the fixed-comparison deviation
\[
    \widehat{\mathcal J}_K(u_{\mathcal F,K})
    -
    \mathcal J_K(u_{\mathcal F,K}).
\]

For the interior residual, set $q_I(x):=|\mathcal L u_{\mathcal F,K}(x)-f(x)|^2 .$ Since \(0\le q_I\le M_I^2\), we have $\operatorname{Var}(q_I(X)) \le M_I^2\,\mu q_I .$ The one-sided Bernstein inequality gives, with probability at least \(1-e^{-t}\),
\[
    (\mu_n-\mu)q_I \le \mu q_I + C\frac{M_I^2t}{n}.
\]

Similarly, for $q_B(y):=|e_{u_{\mathcal F,K}}(y)|^2,$ we have \(0\le q_B\le M_B^2\), and hence with probability at least \(1-e^{-t}\),
\[
    (\sigma_m-\sigma)q_B \le \sigma q_B + C\frac{M_B^2t}{m}.
\]
We now control the spectral boundary part at the fixed comparison function. Define
\[
    \zeta_k(u):=\widehat c_k(u)-c_k(u),\qquad k=1,\ldots,K.
\]
Then
\[
\begin{aligned}
\widehat{\mathcal S}_K(u)-\mathcal S_K(u)
&=
\sum_{k=1}^K
\lambda_k^\beta
\left(
|\widehat c_k(u)|^2-|c_k(u)|^2
\right)
\\
&=
\sum_{k=1}^K
\lambda_k^\beta
\left(
2c_k(u)\zeta_k(u)+|\zeta_k(u)|^2
\right),
\end{aligned}
\]
Since \(\mathbb E\zeta_k(u)=0\), the quadratic
part has expectation
\[
    b_K(u)
    :=
    \sum_{k=1}^K
    \lambda_k^\beta
    \mathbb E|\widehat c_k(u)-c_k(u)|^2
    =
    \frac1m
    \sum_{k=1}^K
    \lambda_k^\beta
    \operatorname{Var}_{\sigma}\big(e_u(Y)\psi_k(Y)\big).
\]
Moreover,
\[
\begin{aligned}
    0\le b_K(u)
    &\le
    \frac1m
    \sum_{k=1}^K
    \lambda_k^\beta
    \sigma(e_u^2\psi_k^2)
    =
    \frac1m\sigma(e_u^2\Theta_K)
    \le
    \frac{\Gamma_K}{m}\sigma|e_u|^2
    \le
    \frac{\Gamma_K}{m}\mathcal J_K(u).
\end{aligned}
\]
The centered part of
\(\widehat{\mathcal S}_K(u_{\mathcal F,K})-\mathcal S_K(u_{\mathcal F,K})\)
is controlled by a Bernstein inequality for the finitely many weighted boundary
coefficients. Hence, with probability at least \(1-e^{-t}\),
\[
    \widehat{\mathcal S}_K(u_{\mathcal F,K}) -
    \mathcal S_K(u_{\mathcal F,K})
    \le
    \mathcal S_K(u_{\mathcal F,K})
    +
    \frac{\Gamma_K}{m}\mathcal J_K(u_{\mathcal F,K})
    +
    C\frac{M_B^2\Gamma_K t}{m}.
\]

Combining the interior, boundary, and spectral estimates, and using a union bound, we obtain
with probability at least \(1-Ce^{-t}\),
\[
\begin{aligned}
\widehat{\mathcal J}_K(u_{\mathcal F,K})
-
\mathcal J_K(u_{\mathcal F,K})
\le
&
\left(1+\frac{\Gamma_K}{m}\right)
\mathcal J_K(u_{\mathcal F,K})
\\
&\quad+
C
\left(
\frac{M_I^2t}{n}
+
\frac{M_B^2(1+\Gamma_K)t}{m}
\right).
\end{aligned}
\]
Combining these estimates gives
\[
\begin{aligned}
\mathcal J_K(\widehat u_{\mathcal F,K})
\le
&
\left(2+\frac{\Gamma_K}{m}\right)
\mathcal J_K(u_{\mathcal F,K})
+
G_K(\widehat u_{\mathcal F,K})
\\
&\quad+
C_\eta
\left(
\frac{M_I^2t}{n}
+
\frac{M_B^2(1+\Gamma_K)t}{m}
\right).
\end{aligned}
\]

Finally, by the definition of \(\tau_K(\mathcal F)\),
\[
    \mathcal J_\infty(u)
    \le
    \mathcal J_K(u)+\tau_K(\mathcal F),
    \qquad u\in\mathcal F.
\]
Applying this inequality to \(u=\widehat u_{\mathcal F,K}\), and then using elliptic
stability, gives
\[
\|\widehat u_{\mathcal F,K}-u^*\|_{\mathcal H^{2s}(\mathcal M^d)}^2
\le
C_{\mathrm{st}}
\left[
\mathcal J_K(\widehat u_{\mathcal F,K})
+
\tau_K(\mathcal F)
\right].
\]
Combining the last two estimates proves the theorem.
\end{proof}

\bibliographystyle{plainnat}    
\bibliography{references}    
\end{document}